%% file: muon_last_iterate.tex
\documentclass[11pt,letterpaper]{article}
\usepackage[margin=1in,headheight=15pt]{geometry}
\usepackage[T1]{fontenc}
\usepackage[utf8]{inputenc}
\usepackage{lmodern}
\usepackage{amsmath,amssymb,amsthm,mathtools}
\usepackage{algorithm,algpseudocode,needspace}
\usepackage{microtype,booktabs,tabularx,array,multirow,enumitem,cite,flafter,xcolor,hyperref}
\hypersetup{colorlinks=true,linkcolor=blue!45!black,citecolor=blue!45!black,
 urlcolor=blue!45!black,
 pdftitle={Last-Iterate Stationarity of Muon with Finitely Many Newton-Schulz Iterations},
 pdfauthor={Ping Li and Vincent Y. F. Tan},bookmarksnumbered=true}
\numberwithin{equation}{section}
\newtheorem{theorem}{Theorem}[section]
\newtheorem{lemma}[theorem]{Lemma}
\newtheorem{proposition}[theorem]{Proposition}
\newtheorem{corollary}[theorem]{Corollary}
\theoremstyle{definition}
\newtheorem{definition}[theorem]{Definition}
\theoremstyle{definition}
\newtheorem{assumption}[theorem]{Assumption}
\theoremstyle{remark}

\newcommand{\EE}{\mathbb E}
\newcommand{\RR}{\mathbb R}
\newcommand{\calF}{\mathcal F}
\newcommand{\ip}[2]{\langle #1,#2\rangle}
\newcommand{\norm}[1]{\lVert #1\rVert}
\newcommand{\frob}[1]{\lVert #1\rVert_{\mathrm F}}
\newcommand{\eps}{\varepsilon}
\newcommand{\Finf}{F_{\inf}}
\newcommand{\ellJ}{\ell_J}
\newcommand{\uJ}{u_J}
\newcommand{\as}{\mathrm{a.s.}}
\DeclareMathOperator{\diag}{diag}
\DeclareMathOperator{\sgn}{sign}

\usepackage{pgfplots}
\pgfplotsset{compat=1.18}

\title{Last-Iterate Stationarity of Muon with Finitely Many Newton-Schulz Iterations}
\author{Ping Li and Vincent Y.~F.~Tan\thanks{The authors are with the Department of Mathematics, National University of Singapore (NUS). V.~Y.~F.~Tan is also with the Department of Electrical and Computer Engineering, NUS. }}
\date{}
\begin{document}
\maketitle

\begin{abstract}
Muon has demonstrated strong empirical performance in large-scale language-model training, motivating a deeper theoretical understanding of its practical implementation. Unlike the idealized polar-factor update, practical Muon employs a fixed number of quintic Newton-Schulz iterations. Existing analyses primarily establish average- or best-iterate stationarity guarantees, while last-iterate convergence remains less understood. In this work, we investigate last-iterate gradient convergence of practical finite-step Muon on smooth, lower-bounded nonconvex objectives.
In the deterministic setting, we show that Muon without input stabilization can exhibit cyclic behavior even on a smooth convex quadratic. In contrast, with positive stabilization, we establish last-iterate gradient convergence and square-summability of gradients for any fixed momentum coefficient below one and a sufficiently small constant learning rate.
In the stochastic setting, we first demonstrate that persistent bounded, unbiased noise can prevent convergence under fixed learning rates and momentum, even with positive stabilization. We then establish almost-sure and mean-square last-iterate gradient convergence under bounded-variance stochastic gradients by combining diminishing learning rates with increasing momentum.
Finally, we construct a smooth counterexample demonstrating that neither convergence theorem implies a pointwise polynomial convergence rate. Nevertheless, we establish density-one convergence rates, guaranteeing polynomial gradient decay for all but an asymptotically vanishing fraction of iterations.
\end{abstract}

\clearpage
\input{sections/01-introduction}

\input{sections/02-related-work}
\input{sections/03-preliminaries}
\input{sections/04-fixed-step-stationarity}
\input{sections/05-stochastic-stationarity}
\input{sections/06-density-one-rates}

\clearpage
\appendix
\input{appendices/a-preliminaries}          
\input{appendices/b-fixed-step-proofs}      
\input{appendices/c-stochastic-proofs}      
\input{appendices/d-density-one-proofs}     

\clearpage
\phantomsection
\addcontentsline{toc}{section}{References}
\bibliographystyle{alpha}
\bibliography{references}
\end{document}

%% file: sections/01-introduction.tex
\section{Introduction}\label{sec:intro}

Muon has emerged as a successful alternative to AdamW for large-scale language-model training. Scaling experiments on Moonlight report roughly twice the training-compute efficiency of AdamW, and the Muon-derived MuonClip optimizer was used to pretrain the trillion-parameter Kimi K2 model \cite{liuScalable2025,kimiK22025}. These empirical successes have motivated growing interest in understanding Muon's optimization dynamics. Muon exploits the matrix structure of neural-network parameters by applying a spectral transformation to momentum-based updates. In its idealized formulation, Muon uses the polar factor of the momentum matrix as a steepest-descent direction under a spectral-norm constraint \cite{bernstein2024,shen2025}. Practical implementations instead use a small, fixed number of quintic Newton--Schulz (NS) iterations, typically with a positive input-normalization stabilizer \cite{jordanBlog,jordanCode,pytorchMuon}. Understanding the convergence behavior of this finite-step update is therefore an important theoretical question.

Recent work has established convergence guarantees for several formulations of Practical Muon. In particular, existing analyses of Muon and related variants have established average- or best-iterate stationarity guarantees for smooth nonconvex objectives \cite{liHong2025,shen2025,kim2026,sato2025,nagashimaIiduka2026,shulginInexact2026,choudhury2026,huangLiMuon2026,peng2026,wang2026}. 
 Such guarantees ensure that a nearly stationary iterate can be found, but they do not rule out large gradients at individual iterations, even late in training. This motivates studying \emph{last-iterate stationarity}, which requires the gradient at the current iterate to vanish as training progresses:
\begin{equation*}
	\|\nabla F(W_t)\| \longrightarrow 0
	\quad \text{as } t\to\infty,
\end{equation*}
where $F$ denotes the original objective function and $W_t$ the matrix-valued parameters at iteration $t$. Related last-iterate gradient convergence guarantees have been established for Soft Muon under additional regularity and parameter restrictions \cite{oliveira2026}. Last-iterate objective-gap bounds have also been obtained for variance-reduced Muon variants under Polyak--\L{}ojasiewicz conditions \cite{chang2026} and adaptively scaled Muon variants under star-convexity \cite{demidovich2026}. However, whether the original Muon algorithm admits last-iterate stationarity guarantees in practical remains an open question.

A second important distinction concerns how gradients are obtained. In deterministic optimization, each update uses the exact gradient of the objective, whereas stochastic optimization relies on mini-batch gradient estimates, whose noise may persist when the batch size remains fixed \cite{cutkosky2020,liu2024}. Even with exact gradients, normalized updates with a constant learning rate can exhibit persistent oscillations near stationary points, as illustrated by quadratic analyses of Muon \cite{gonon2026}. With stochastic gradients, persistent noise introduces an additional source of instability, and nonconvergence has been demonstrated for finite-NS Muon under certain noise distributions \cite{do2026}.
These observations highlight two distinct challenges for last-iterate stationarity of finite-step Muon: instability arising from normalized updates and persistent fluctuations introduced by stochastic gradients. The deterministic setting allows us to isolate the first challenge, while the stochastic setting requires controlling both effects. We investigate how these challenges can be addressed to establish last-iterate stationarity, and further examine what convergence rates can be guaranteed for individual iterates.

\paragraph{Contributions.} Our main contributions are summarized below.
\begin{itemize}
    \item \textbf{Last-iterate convergence with deterministic gradients.}
    We first show that without input stabilization ($\eps=0$), finite-NS Muon can cycle even on a smooth convex quadratic (Proposition~\ref{prop:cycle}).
    With positive stabilization, as used in practice~\cite{jordanCode,pytorchMuon}, we establish last-iterate gradient convergence for any fixed momentum coefficient $\beta<1$ and a sufficiently small constant learning rate (Theorem~\ref{thm:fixed}).
    The result also guarantees square-summable gradients and extends to square-summable perturbations (Appendix~\ref{app:fixed-perturbations}).

    \item \textbf{Last-iterate convergence with stochastic gradients.}
    We show that even with positive stabilization, fixed learning rates and momentum can fail to converge under persistent bounded, unbiased stochastic noise (Proposition~\ref{prop:noise}).
    To address this, we combine decreasing learning rates with increasing momentum and establish almost-sure and mean-square last-iterate gradient convergence under conditionally unbiased, bounded-variance stochastic gradients (Theorem~\ref{thm:stoch}), building on stochastic momentum and tracking techniques.

    \item \textbf{Density-one convergence rates.}
    We construct a smooth objective under the settings of both convergence theorems for which no pointwise polynomial gradient rate holds (Proposition~\ref{prop:stoch-no-rate}).
    We then establish \emph{density-one rates}, guaranteeing polynomial gradient decay outside a set of iterations whose long-run proportion vanishes (Theorem~\ref{thm:density-rates}).
     These rates describe almost all late iterations along an individual run, providing a useful quantitative view of training progress. 
\end{itemize}

To the best of our knowledge, this is the first work to establish last-iterate gradient convergence for practical finite-step Muon under general smooth, lower-bounded nonconvex objectives, without additional structural assumptions such as the Polyak--\L{}ojasiewicz condition or star-convexity. Table~\ref{tab:closest-comparison} compares our deterministic and stochastic settings with existing results in terms of the Muon variants considered, assumptions on the objective function, and convergence guarantees.
 The first block summarizes average- or best-iterate gradient bounds, while the second focuses on last-iterate guarantees (including objective-value results) and a related training-loss hitting-time bound. The final two rows present our deterministic and stochastic results. This grouping separates guarantees for averages or selected iterates from those that characterize later iterations.
 Section~\ref{sec:related-work} gives a more detailed comparison with the related literature.
\begin{table}[!htb]
\centering
\caption{Convergence guarantees for Muon and related methods.}
\label{tab:closest-comparison}
\begingroup
\fontsize{11pt}{13.2pt}\selectfont
\setlength{\tabcolsep}{2.3pt}
\renewcommand{\arraystretch}{1.10}
\makebox[\textwidth][c]{%
\begin{minipage}{\linewidth}
\begin{tabularx}{\linewidth}{@{}lll>{\raggedright\arraybackslash}Xl@{}}
\toprule
				\textbf{Work} & \textbf{Type} & \textbf{Method} & \textbf{Nonconvex} & \textbf{Convergence} \\
				\midrule
				Li--Hong~\cite{liHong2025} & Stoch. & Polar & $\checkmark$ & Avg. \\
				Kim--Oh~\cite{kim2026} & Stoch. & Finite-NS (Taylor) & $\checkmark$ & Avg. \\
				Sato et al.~\cite{sato2025} & Stoch. & Polar (Nest., WD) & $\checkmark$ & Avg. \\
				Nagashima--Iiduka~\cite{nagashimaIiduka2026} & Stoch. & Polar (Nest.) & $\checkmark$ & Best \\
				Shulgin et al.~\cite{shulginInexact2026} & Det./Stoch. & Inexact polar (LMO) & $\checkmark$ & Avg./Best \\
				Choudhury et al.~\cite{choudhury2026} & Stoch. & Inexact polar (Nest.) & $\checkmark$ & Best \\
				Huang et al.~\cite{huangLiMuon2026} & Stoch. & LiMuon (VR, NS) & $\checkmark$ & Avg. \\
				Peng et al.~\cite{peng2026} & Stoch. & Practical (Nest.) & $\checkmark$ & Avg. \\
				Wang et al.~\cite{wang2026} & Stoch. & Practical (WD) & $\checkmark$ & Avg. \\
				\midrule
				De Oliveira et al.~\cite{oliveira2026} & Det. & Soft Muon & PL & Last / linear obj. \\
				Chang et al.~\cite{chang2026} & Stoch. & Muon--MVR & $\checkmark$(avg.)/PL(last) & Avg. / last obj. \\
				Demidovich et al.~\cite{demidovich2026} & Det. & Scale-Calibrated & Star-convex& Last obj. gap \\
				Roy--Chakraborty~\cite{royChakraborty2026} & Det. & Practical & Squared loss  & Loss hitting time \\
				\midrule
				\textbf{This work} & \textbf{Det.} & \textbf{Practical} & $\checkmark$ & \textbf{Last} \\
				\textbf{This work} & \textbf{Stoch.} & \textbf{Practical} & $\checkmark$ & \textbf{Last} \\
				\bottomrule
			\end{tabularx}
			\vspace{2pt}
            
			     Note. \emph{Type: }Det. = deterministic gradients; Stoch. = stochastic gradients.
			\emph{Method:} NS = Newton--Schulz; WD = weight decay; VR = variance reduction; PL = Polyak--\L{}ojasiewicz; LMO = linear minimization oracle;
			a.s. = almost surely; M.S. = mean-square.
			Polar Muon uses the exact polar factor; Practical Muon uses tuned quintic finite-NS. Nest. = Nesterov-style momentum; parentheses identify variants.
			\emph{Nonconvex:} $\checkmark$ indicates no additional structural loss assumption beyond smoothness and lower boundedness.
			\emph{Convergence:} Avg., Best, and Last refer to gradient stationarity unless objective gap or loss is specified. 
		\end{minipage}%
}
\endgroup
\end{table}

The remainder of the paper is organized as follows.
Section~\ref{sec:related-work} reviews the related literature.
Section~\ref{sec:setup} introduces the practical finite-NS update and its spectral properties.
Section~\ref{sec:fixed} studies last-iterate convergence with deterministic gradients.
Section~\ref{sec:stoch} extends the analysis to persistent stochastic noise through time-varying learning rates and momentum.
Section~\ref{sec:rates} examines the limitations of pointwise polynomial rates and establishes density-one convergence rates.
Detailed proofs are provided in the appendices.

%% file: sections/02-related-work.tex
\section{Related Work}\label{sec:related-work}

\paragraph{Matrix geometry and linear minimization oracles.}
Muon is closely related to spectral-norm steepest descent and matrix-aware optimization~\cite{carlson2015,gupta2018,bernstein2024}. In particular, its idealized polar-factor update can be interpreted as a linear minimization oracle (LMO) over a spectral-norm ball. This perspective has been further developed in the context of non-Euclidean and layer-wise optimization~\cite{pethickLMO2025,riabinin2026}, including analyses under generalized smoothness conditions tailored to neural-network parameter blocks.

\paragraph{Exact, inexact, and finite-step orthogonalization.}
Analyses of exact-polar Muon and momentum variants have obtained average- or best-iterate stationarity bounds for smooth nonconvex objectives~\cite{liHong2025,shen2025,sato2025,nagashimaIiduka2026}. Other work considers approximate polar computations: Kim and Oh~\cite{kim2026} study Newton--Schulz approximation quality, Shulgin et al.~\cite{shulginInexact2026} formulate an inexact-LMO framework, and Choudhury et al.~\cite{choudhury2026} analyze inexact polar updates with Nesterov momentum. For practical finite-step implementations, Peng et al.~\cite{peng2026} investigate quintic Newton--Schulz updates under layer-wise generalized smoothness, while Wang et al.~\cite{wang2026} incorporate the tuned quintic map and weight decay. Li and Tsuchiya~\cite{liTsuchiya2026} instead exploit the smoothing effect of finite-depth Newton--Schulz iterations in nonsmooth optimization. These studies differ both in their treatment of finite orthogonalization and in the stationarity criteria they establish.

\paragraph{Last-iterate guarantees.}
De Oliveira et al.~\cite{oliveira2026} establish full-sequence stationarity results for bounded-preconditioner momentum methods, including a Soft Muon specialization under additional structural assumptions such as a global Polyak--\L{}ojasiewicz condition; the stated Soft Muon convergence corollary imposes $\beta<1/\sqrt{2}$. Other Muon-related last-iterate results concern objective values rather than gradient stationarity: Chang et al.~\cite{chang2026} derive objective-gap guarantees under a Polyak--\L{}ojasiewicz condition, and Demidovich et al.~\cite{demidovich2026} consider an adaptively scaled variant under star-convexity. Roy and Chakraborty~\cite{royChakraborty2026} study a training-loss hitting-time criterion in a structured squared-loss setting. These criteria are distinct from convergence of the gradient norm at every sufficiently late iterate.

\paragraph{Stochastic momentum and noise models.}
Last-iterate convergence has also been investigated beyond Muon. Liu and Yuan~\cite{liu2024} analyze almost-sure convergence for stochastic gradient methods with heavy-ball and Nesterov momentum, and Liang and Xu~\cite{liangXu2026} develop guarantees for stochastic adaptive methods. Related normalized-momentum analyses use gradient-tracking and descent arguments~\cite{cutkosky2020,oik2025,oik2026}. The stochastic oracle assumptions vary: Shulgin et al.~\cite{shulgin2026} analyze LMO-based methods under expected-smoothness/ABC conditions, while variance-reduced Muon variants seek improved finite-time stationarity complexity~\cite{qian2025,huangLiMuon2026}. The distinction between persistent bounded-variance noise and variance-reduced estimators is particularly relevant to asymptotic convergence.

\paragraph{Nonconvergence and stabilization.}
Normalized updates can remain nonvanishing near stationary points, potentially producing oscillations. Quadratic Muon dynamics have been examined by Gonon et al.~\cite{gonon2026}, and Do et al.~\cite{do2026} exhibit stochastic nonconvergence phenomena for generalized finite-NS Muon. Parshakova et al.~\cite{parshakova2026} demonstrate limitations of Muon on convex Lipschitz objectives and study error feedback as a remedy; related error-feedback methods have also been considered for compressed LMO optimization~\cite{gruntkowska2025}. The underlying obstructions depend on the objective class, normalization, noise model, and parameter schedule.

\paragraph{Pointwise and density-one rates.}
Classical smooth nonconvex analysis often bounds an average gradient norm or selects a nearly stationary iterate~\cite{ghadimi2013}, which does not imply a pointwise polynomial rate along an entire trajectory. Vernimmen and Glineur~\cite{vernimmen2026} study last-iterate gradient-norm limitations and density-one guarantees for horizon-independent methods in smooth convex optimization. Such results distinguish guarantees valid at every iteration from those valid outside a set of asymptotically vanishing density.

%% file: sections/03-preliminaries.tex
\section{Preliminaries}\label{sec:setup}

\subsection{Problem Setup and Muon Update}
We consider a single matrix parameter $W\in\RR^{m\times n}$ and a continuously differentiable loss $F:\RR^{m\times n}\to\RR$. Throughout, $\norm{X}:=\frob X$ and $\ip{X}{Y}:=\operatorname{tr}(X^\top Y)$ denote the Frobenius norm and inner product, and $\EE$ denotes expectation. Throughout the paper, we impose the following assumption on the objective function.

\begin{assumption}[Global smoothness and lower boundedness]\label{ass:smooth}
There exist $L>0$ and $\Finf\in\RR$ such that, for all $X,Y\in\RR^{m\times n}$,
\begin{equation*}
 F(X)\ge\Finf,\quad
 \norm{\nabla F(X)-\nabla F(Y)}\le L\norm{X-Y}.
\end{equation*}
\end{assumption}

Let $g_t=\nabla F(W_t)$ and let $G_t=g_t+\zeta_t$ be a possibly perturbed gradient. Given deterministic sequences $\eta_t>0$, $\beta_t\in[0,1)$, and $\theta_t\in[0,1]$, the finite-step Muon update is
\begin{align}
 M_t&=\beta_tM_{t-1}+(1-\beta_t)G_t,\label{eq:ema}\\
 N_t&=\theta_tM_t+(1-\theta_t)G_t,\label{eq:mix}\\
 D_t&=\Phi_\eps(N_t),\notag\\
 W_{t+1}&=W_t-\eta_tD_t.\label{eq:update}
\end{align}
Here $W_1$ and $M_0$ are deterministic and finite. The matrix $M_t$ is an exponential moving average, while $N_t$ is the input to the Newton--Schulz transformation. The choices $\theta_t=1$ and $\theta_t=\beta_t$ give, respectively, EMA input and the Nesterov-style mixture of the current gradient and momentum before the NS transformation used in the reference Muon implementation \cite{jordanCode}.

For fixed $J\ge1$ and $\eps\ge0$, let $\Phi_\eps(N)$ be the output $Z_J$ of the following Newton--Schulz iteration:
\begin{equation}\label{eq:ns}
 Z_0=\frac{N}{\frob{N}+\eps},\quad
 Z_{j+1}=aZ_j+bZ_jZ_j^\top Z_j+c(Z_jZ_j^\top)^2Z_j,
 \quad j=0,\dots,J-1,
\end{equation}
We use the practical quintic coefficients
\begin{equation}\label{eq:coeff}
 (a,b,c)=(3.4445,-4.7750,2.0315).
\end{equation}
We set $\Phi_0(0)=0$. The coefficients, the common choice $J=5$, and the input normalization with a small positive $\eps$ are motivated by practical Muon implementations \cite{jordanBlog,jordanCode,pytorchMuon,peng2026,wang2026}. 

\subsection{Properties of the Finite Newton--Schulz Map}
We next establish bounds on the magnitude of the finite Newton--Schulz output and its alignment with the input. These properties provide the basic estimates used in the convergence analysis.

The iteration in \eqref{eq:ns} acts on each singular value of the normalized input through the scalar polynomial $P$. Let $\varphi_J$ denote its $J$-fold composition, and let $q_J$ denote the resulting amplification factor:
\begin{equation*}
 P(s)=as+bs^{3}+cs^{5},\quad
 \varphi_J=P^{\circ J},\quad
 q_J(s)=\begin{cases}\varphi_J(s)/s,&s>0,\\a^J,&s=0,\end{cases}
\end{equation*}
Since every singular value of the normalized input lies in $[0,1]$, define
\begin{equation*}
 \ellJ=\min_{0\le s\le1}q_J(s),\quad
 \uJ=\max_{0\le s\le1}q_J(s).
\end{equation*}
These constants depend only on the polynomial coefficients and the iteration count $J$, and control the amplification of all normalized input singular values.

\begin{lemma}[Finite-NS bounds and alignment]\label{lem:map}
For each fixed finite $J\ge1$, $0<\ellJ\le\uJ<\infty$. For any matrix $N$, write $R=\frob N$. Then
\begin{align}
 \ellJ\frac{R}{R+\eps}&\le\frob{\Phi_\eps(N)}
 \le\uJ\frac{R}{R+\eps}\le\uJ,\label{eq:radial}\\
 \ip{N}{\Phi_\eps(N)}_{\mathrm F}&\ge
 \ellJ\frac{R^2}{R+\eps}.\label{eq:align}
\end{align}
If $\eps>0$, additionally
\begin{equation}\label{eq:sector}
 \ip{N}{\Phi_\eps(N)}_{\mathrm F}
 \ge\frac{\eps}{\uJ}\frob{\Phi_\eps(N)}^2.
\end{equation}
When $R=\eps=0$, fractions are interpreted as zero.
\end{lemma}

The radial bounds \eqref{eq:radial} control the magnitude of the transformed matrix, while \eqref{eq:align} guarantees its alignment with the input. When $\eps>0$, the additional sector bound \eqref{eq:sector} controls the squared update norm through this alignment and is central to the fixed-step Lyapunov analysis. These estimates hold for every fixed finite $J$ and do not require the transformed singular values to be close to one. Thus, the convergence analysis relies on uniform spectral bounds rather than accurate approximation of the polar factor. The proof appears in Appendix~\ref{app:proof-lem-map}.

To express the subsequent descent estimates compactly, define
\begin{equation*}
 B:=\uJ,\quad r_t:=N_t-g_t,\quad
 \psi_\eps(x):=\begin{cases}x^2/(x+\eps),&x+\eps>0,\\0,&x=\eps=0.\end{cases}
\end{equation*}
Here $B$ is a uniform bound on the update norm, so that $\norm{D_t}\le B$; $r_t$ measures the discrepancy between the NS input and the true gradient; and $\psi_\eps$ captures the radial dependence of the alignment bound. In particular, $\psi_0(x)=x$ for $x\ge0$.

%% file: sections/04-fixed-step-stationarity.tex
\section{Last-Iterate Convergence in the Deterministic Setting}
\label{sec:fixed}
We first consider exact gradients to isolate the effect of input normalization. Without stabilization, normalization removes the magnitude of every nonzero scalar input, so the update need not vanish near a stationary point. The following example shows that this can produce a persistent two-cycle for any positive constant step size, even with fixed momentum. Its proof appears in Appendix~\ref{app:cycle-proof}.

\begin{proposition}[A parameter two-cycle when $\eps=0$]\label{prop:cycle}
Take $F(w)=\frac12 w^2$, exact gradients, $\eps=0$, fixed $\eta>0$, $\beta\in[0,1)$, and $\theta\in[0,1]$. Let $c_J=\varphi_J(1)>0$ and $\delta=\eta c_J$. With $w_1=\delta/2$ and $M_0=0$, the iterates satisfy
\begin{equation*}
 w_t=(-1)^{t-1}\frac\delta2,
 \quad|\nabla F(w_t)|=\frac{\delta}{2}>0\quad(t\ge1).
\end{equation*}
\end{proposition}

Positive stabilization changes this behavior: the transformed direction vanishes as the input approaches zero, and the sector bound \eqref{eq:sector} relates its squared norm to its alignment with the input. Together with control of the momentum tracking error, this yields stationarity for a sufficiently small constant step size.

For the remainder of this section, assume exact gradients and fix $\eps>0$, $\beta_t\equiv\beta\in[0,1)$, $\theta_t\equiv\theta\in[0,1]$, and $\eta_t\equiv\eta>0$. Define
\begin{equation}\label{eq:constants}
 \alpha:=1-\beta,\quad \kappa:=\frac{\eps}{\uJ},\quad
 e_t:=M_t-g_t,\quad
 N_t-g_t=\theta e_t\quad (G_t=g_t).
\end{equation}
We impose the following sufficient condition on the constant step size:
\begin{equation}\label{eq:fixed-stepsize}
 0<\eta\le\frac{\eps(1-\beta)}{4L\uJ}
 =\frac{\kappa\alpha}{4L}.
\end{equation}

\begin{theorem}[Fixed-step full-sequence stationarity: exact gradients]\label{thm:fixed}
Suppose Assumption~\ref{ass:smooth} holds. In the deterministic setting where \(G_t=g_t\), fix finite \(J\ge1\), \(\eps>0\), \(\beta\in[0,1)\), and \(\theta\in[0,1]\), and let the constant step size satisfy \eqref{eq:fixed-stepsize}. Then
\begin{equation}\label{eq:fixed-result}
 \sum_{t=1}^{\infty}\norm{D_t}^2<\infty,\quad
 \sum_{t=1}^{\infty}\norm{M_t-g_t}^2<\infty,\quad
 \sum_{t=1}^{\infty}\norm{g_t}^2<\infty.
\end{equation}
In particular, \(\norm{\nabla F(W_t)}\to0\).
\end{theorem}

The bound \eqref{eq:fixed-stepsize} guarantees stationarity from arbitrary finite initialization, but it can be conservative. Its derivation controls gradient variation using the global smoothness constant $L$ and bounds the interaction between the update and the tracking error uniformly along the trajectory. As $\beta$ approaches one, the tracking contraction weakens, and this argument requires a smaller step size.
Near a stationary point, a linearization can account more precisely for the interaction between the parameter and momentum. For the scalar quadratic $F(w)=\lambda w^2/2$ with $\lambda>0$, the stabilized finite-NS map satisfies $\Phi_\eps(n)=(a^J/\eps)n+o(|n|)$ near the origin. The eigenvalues of the joint parameter--momentum linearization then give an explicit local stability range that can be substantially larger than \eqref{eq:fixed-stepsize}. For Nesterov-style mixing $\theta=\beta$, Proposition~\ref{prop:schur} gives local asymptotic stability whenever
\begin{equation*}
 0<\eta<\eta_{\mathrm{loc}}
 :=\frac{2(1+\beta)\eps}
 {\lambda a^J(1-\beta)(1+2\beta)}.
\end{equation*}
By comparison, the global sufficient threshold for the same quadratic, with $L=\lambda$, is $\eta_{\mathrm{glob}}:=\eps(1-\beta)/(4\lambda\uJ)$. Since $\uJ\ge q_J(0)=a^J$, their ratio satisfies
\begin{equation*}
 \frac{\eta_{\mathrm{loc}}}{\eta_{\mathrm{glob}}}
 =\frac{8\uJ}{a^J}\,
 \frac{1+\beta}{(1-\beta)^2(1+2\beta)}
 \ge\frac{8(1+\beta)}{(1-\beta)^2(1+2\beta)}.
\end{equation*}
Consequently, the ratio of the local and global step-size thresholds grows as $(1-\beta)^{-2}$ when $\beta\uparrow1$. The larger local range guarantees convergence only for initial parameter--momentum states sufficiently close to the quadratic stationary point; it does not replace the global guarantee of Theorem~\ref{thm:fixed} or imply a faster convergence rate. The exact local calculation, including general $\theta$, is given in Appendix~\ref{app:obstructions}. The extension of Theorem~\ref{thm:fixed} to square-summable gradient perturbations, including growing-batch sampling, is given in Appendix~\ref{app:fixed-perturbations}.

\subsection{Proof Sketch: Coupled Descent–Tracking Lyapunov Analysis}
\label{subsec:fixed-proof-analysis}
The sector bound controls alignment with $N_t$, whereas objective descent depends on alignment with the true gradient $g_t$. Since $N_t-g_t=\theta e_t$, this discrepancy contributes a tracking-error term to the descent estimate. Combining smoothness with the sector bound gives the following inequality, proved in Appendix~\ref{app:proof-lem-fixed-descent}.

\begin{lemma}[Objective descent]\label{lem:fixed-descent}
Under the fixed-step setting above,
\begin{equation}\label{eq:fixed-descent}
 F(W_{t+1})-F(W_t)
 \le-\left(\frac{\eta\kappa}{2}-\frac{L\eta^2}{2}\right)
 \norm{D_t}^2
 +\frac{\eta}{2\kappa}\norm{e_t}^2.
\end{equation}
\end{lemma}

The positive term in \eqref{eq:fixed-descent} prevents this estimate alone from guaranteeing descent. To control it, we use the contraction of the momentum tracking error, allowing for changes in the gradient as the parameter is updated. The following bound follows from the EMA recursion; its proof appears in Appendix~\ref{app:proof-lem-tracking}.

\begin{lemma}[Momentum tracking]\label{lem:tracking}
Under the same conditions,
\begin{equation}\label{eq:tracking}
 \norm{e_{t+1}}^2
 \le\frac{1+\beta}{2}\norm{e_t}^2
 +\frac{\beta(1+\beta)L^2\eta^2}{1-\beta}\norm{D_t}^2.
\end{equation}
\end{lemma}

Since $(1+\beta)/2<1$, the tracking error contracts, but each parameter update introduces an additional error through gradient variation. We therefore combine the two estimates in a weighted potential, choosing the weight so that the tracking contraction absorbs the positive error term in \eqref{eq:fixed-descent}. Set
\begin{equation*}
 V_t:=F(W_t)+\lambda\norm{e_t}^2,
 \quad\lambda:=\frac{2\eta}{\kappa(1-\beta)}.
\end{equation*}
With this choice, the coefficient of $\norm{e_t}^2$ after addition is $\eta/(2\kappa)-\lambda(1-\beta)/2=-\eta/(2\kappa)$; condition~\eqref{eq:fixed-stepsize} also absorbs the $\norm{D_t}^2$ cost in Lemma~\ref{lem:tracking}. Hence
\begin{equation}\label{eq:fixed-lyap-main-bound}
 V_{t+1}-V_t\le
 -\frac{\eta\kappa}{4}\norm{D_t}^2
 -\frac{\eta}{2\kappa}\norm{e_t}^2.
\end{equation}
The potential therefore decreases by an amount that controls both squared norms. The coefficient calculation appears in Appendix~\ref{app:proof-lem-lyap}.

Since $V_t$ is nonincreasing and bounded below by $\Finf$, telescoping \eqref{eq:fixed-lyap-main-bound} gives square-summability of both $D_t$ and $e_t$. Moreover, $F(W_t)\le V_t\le V_1$; smoothness and the lower bound on $F$ imply $\norm{g_t}^2\le 2L(F(W_t)-\Finf)\le 2L(V_1-\Finf)$. Hence $N_t=g_t+\theta e_t$ is also bounded. The lower radial bound in \eqref{eq:radial} then yields square-summability of $N_t$, and the identity $g_t=N_t-\theta e_t$ gives the claimed square-summability of the gradients. In particular, $\norm{g_t}\to0$, proving Theorem~\ref{thm:fixed}.
The complete proof appears in Appendix~\ref{app:proof-thm-fixed}. Section~\ref{sec:rates} further quantifies this convergence through a density-one rate.

\Needspace{12\baselineskip}

%% file: sections/05-stochastic-stationarity.tex
\section{Last-Iterate Convergence under Stochastic Noise}
\label{sec:stoch}
We now turn to stochastic gradients with persistent, bounded-variance noise. In this setting, positive input stabilization alone does not guarantee last-iterate stationarity when both the learning rate and the momentum coefficient are fixed. We first demonstrate this obstruction with a quadratic counterexample. We then show that suitably coordinating decreasing learning rates with momentum coefficients approaching one ensures that the true gradient converges to zero almost surely and in mean square.

We start with a quadratic counterexample in which even symmetric, unbiased noise maintains a nonzero gradient error under fixed hyperparameters; its proof appears in Appendix~\ref{app:noise-proof}.

\begin{proposition}[Persistent symmetric noise at a fixed step]
\label{prop:noise}
Let $F(w)=\frac12 w^2$ and $G_t=w_t+\zeta_t$, where $\zeta_t$ are independent equiprobable variables in $\{+\sigma,-\sigma\}$, with $\sigma>0$. Fix $\eta>0$, $\beta\in[0,1)$, $\theta\in[0,1]$, and $\eps\ge0$. Put
\[
 a_0=(1-\theta\beta)\sigma>0,\quad
 d_0=\ellJ\frac{a_0}{a_0+\eps}>0.
\]
Then
\begin{equation}\label{eq:noise-floor}
 \limsup_{t\to\infty}\EE|\nabla F(w_t)|
 \ge\frac{\eta d_0}{4}>0.
\end{equation}
Moreover,
\begin{equation}\label{eq:bad-probability}
 \limsup_{t\to\infty}
 \mathbb P\bigl(|\nabla F(w_t)|\ge\eta d_0/2\bigr)
 \ge\frac14.
\end{equation}
In particular, the gradient does not converge to zero in probability.
\end{proposition}

The obstruction identifies the need to control how noise accumulates through the optimizer, rather than relying on positive stabilization alone. We do so with \emph{decreasing learning rates} and \emph{increasing momentum}. The latter averages perturbations over progressively longer windows, while the former limits gradient drift within that memory.

Let $(\calF_t)_{t\ge0}$ be a filtration with $W_t,M_{t-1}$ measurable with respect to $\calF_{t-1}$, and assume
\begin{equation}\label{eq:MDS}
 \EE[\zeta_t\mid\calF_{t-1}]=0,\quad
 \EE[\norm{\zeta_t}^2\mid\calF_{t-1}]\le\sigma^2
 \quad\text{a.s.}
\end{equation}
The noise can be asymmetric and persists at bounded variance. Choose the explicit, horizon-free schedule
\begin{equation}\label{eq:schedule}
 \eta_t=\eta_0t^{-p},\quad
 \beta_t=1-\alpha_0t^{-q},
 \quad \eta_0>0,\quad0<\alpha_0\le1,
\end{equation}
with exponents in the admissible region
\begin{equation}\label{eq:exponents}
 \frac34<p\le1,\quad 2(1-p)<q<2p-1.
\end{equation}

\begin{theorem}[Full-sequence stationarity under persistent noise]\label{thm:stoch}
Suppose Assumption~\ref{ass:smooth} and \eqref{eq:MDS} hold. Fix finite $J\ge1$ and $\eps\ge0$; use the schedules \eqref{eq:schedule}--\eqref{eq:exponents}, with either $\theta_t=1$ or $\theta_t=\beta_t$ for every $t$. Then
\begin{equation}\label{eq:stoch-conclusion}
 \norm{g_t}\longrightarrow0\quad\as,\quad
 \EE\norm{g_t}^2\longrightarrow0.
\end{equation}
In addition,
\begin{equation}\label{eq:stoch-weighted-main}
 \sum_{t=1}^{\infty}\eta_t\psi_\eps(\norm{g_t})<\infty,
 \quad
 \sum_{t=1}^{\infty}\eta_t\norm{g_t}^2<\infty
 \quad\text{a.s.}
\end{equation}
The convergence and energy guarantees hold for the entire admissible range $3/4<p\le1$.
\end{theorem}

Theorem~\ref{thm:stoch} shows that, despite persistent bounded-variance gradient noise, decreasing learning rates and increasing momentum yield \emph{full-sequence} stationarity, both almost surely and in mean square, together with the weighted summability guarantees in \eqref{eq:stoch-weighted-main}. The mixing coefficient $\theta_t$ determines the input to the finite NS map. For $\theta_t=1$, we have $N_t=M_t$, so the direction is computed from the exponential moving average alone (ordinary momentum). For $\theta_t=\beta_t$, we instead have $N_t=\beta_tM_t+(1-\beta_t)G_t =\beta_t^2M_{t-1}+(1-\beta_t^2)G_t$, the Nesterov-style mixture used in practical Muon. Relative to $\theta_t=1$, this gives the current stochastic gradient more weight while retaining the noise-filtering effect of momentum. The convergence guarantees thus apply to both ordinary momentum and Nesterov-style mixing under persistent, bounded-variance gradient noise.

\subsection{Proof sketch: Momentum Filtering and Convergence Analysis}
\label{subsec:stoch-proof-analysis}

The proof of Theorem~\ref{thm:stoch} hinges on controlling the \emph{mixed-input tracking error} $r_t=N_t-g_t$. Although the stochastic gradient noise need not decay, the increasing momentum coefficient makes the EMA average perturbations over an effective memory of order $t^q$. This reduces the noise contribution to the expected tracking error to order $t^{-q/2}$. At the same time, the parameters continue to move, so the momentum estimate must track a changing gradient. With steps of order $t^{-p}$, this drift introduces an additional lag of order $t^{q-p}$. For both ordinary momentum ($\theta_t=1$) and Nesterov-style mixing ($\theta_t=\beta_t$), the tracking analysis in Appendix~\ref{app:proof-lem-ema-tracking} gives
\begin{equation}\label{eq:tracking-schedule-bound}
 \EE\norm{r_t}=O(t^{-q/2}+t^{q-p}),\quad
 \eta_t\EE\norm{r_t}=O(t^{-p-q/2}+t^{q-2p}).
\end{equation}

The two terms in \eqref{eq:tracking-schedule-bound} determine how the learning rate and momentum should be coordinated. In particular, their learning-rate-weighted sum is finite when $p+q/2>1$ and $2p-q>1$, exactly the inequalities defining the admissible region \eqref{eq:exponents}. The same estimate also identifies a natural balanced schedule: equating the noise-averaging and gradient-drift exponents gives
\begin{equation}\label{eq:balanced-schedule}
 \frac q2=p-q\quad\Longleftrightarrow\quad
 q=\frac{2p}{3},\quad
 \EE\norm{r_t}=O(t^{-p/3}).
\end{equation}
Thus $\eta_t\EE\norm{r_t}=O(t^{-4p/3})$, which is summable for $p>3/4$. This balance explains how growing momentum filters persistent noise without allowing its tracking lag to accumulate.

We now connect the tracking estimate to descent of the objective. The finite-NS alignment inequality and smoothness imply the inexact-input descent bound (proved in Appendix~\ref{app:proof-lem-inexact})
\[
\begin{aligned}
 F(W_{t+1})-F(W_t)\le{}&
 -\ellJ\eta_t\psi_\eps(\norm{g_t})
 +(\ellJ+B)\eta_t\norm{r_t}
 +\frac{LB^2}{2}\eta_t^2.
\end{aligned}
\]
By the exponent conditions, $\sum_t\eta_t\EE\norm{r_t}<\infty$; Tonelli's theorem therefore gives $\sum_t\eta_t\norm{r_t}<\infty$ almost surely. Combining this with $\sum_t\eta_t^2<\infty$ and the lower bound on $F$, we can sum the descent inequality to obtain
\begin{equation}\label{eq:stoch-energy-explain}
 \sum_{t=1}^{\infty}\eta_t\psi_\eps(\norm{g_t})<\infty
 \quad\text{a.s.}
\end{equation}

The remaining step converts this accumulated gradient control into convergence of the entire sequence. Because the finite-NS directions are bounded, smoothness yields $|\norm{g_{t+1}}-\norm{g_t}|\le LB\eta_t$. Together with $\sum_t\eta_t=\infty$, this temporal regularity upgrades \eqref{eq:stoch-energy-explain} to $\norm{g_t}\to0$ almost surely. Finally, an integrable objective-gap bound gives $\EE\norm{g_t}^2\to0$. The complete proof of Theorem~\ref{thm:stoch} is given in Appendix~\ref{app:proof-thm-stoch}.

%% file: sections/06-density-one-rates.tex
\Needspace{12\baselineskip}
\section{Density-One Rates Beyond Pointwise Guarantees}
\label{sec:rates}
The preceding sections establish convergence to stationarity but leave open whether the gradient norms decay at a polynomial rate. To address this question, we construct a smooth counterexample showing that sparse gradient peaks can rule out any pointwise polynomial rate in both convergence regimes, even with exact gradients. This motivates a weaker rate notion that allows an asymptotically vanishing fraction of exceptional iterations. Using the gradient-energy bounds established above, we derive polynomial rates in this density-one sense.

\begin{figure}[t]
    \centering
    \input{f_and_h_tikz.tex}
    \caption{Schematic illustration of an objective with sparse gradient
    peaks; horizontal distances are not to scale.}
    \label{fig:sparse-gradient-peaks}
\end{figure}
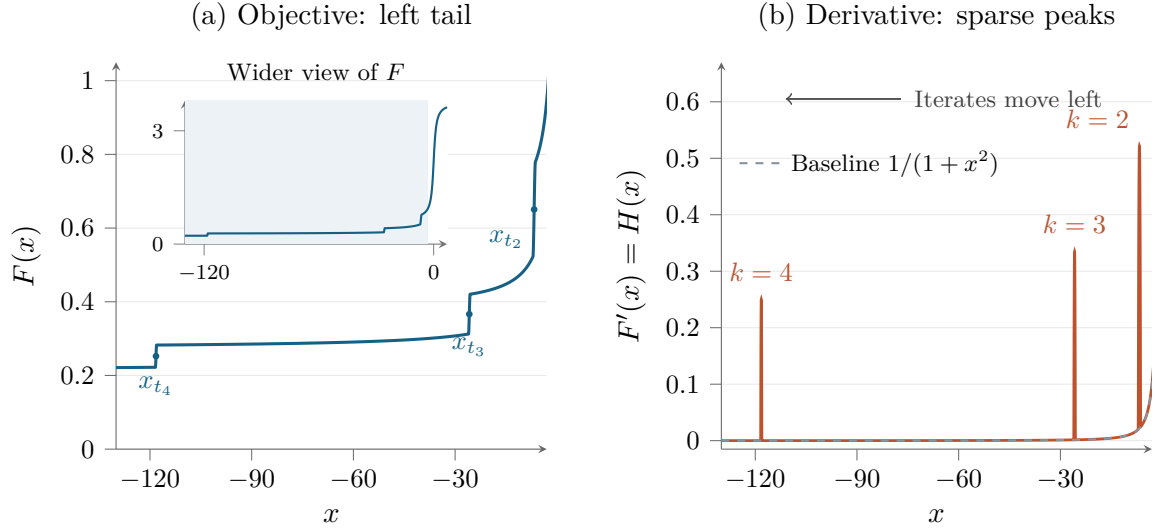
\subsection{Sparse Peaks and Pointwise-Rate Barriers}
\label{subsec:rate-separation}
We construct a one-dimensional objective for which Muon converges to stationarity in both regimes, yet $|F'(w_t)|$ is not $O(t^{-\nu})$ for any $\nu>0$. The idea is to place gradient bumps of height $1/k$ increasingly far apart, so that the algorithm encounters these small but nonzero gradients at times that grow too rapidly for a polynomial rate.

To specify the bump locations, fix $\eta_0>0$, let $c_J:=\varphi_J(1)>0$, and define an auxiliary sequence of positions $x_t:=-c_J\eta_0\sum_{s=1}^{t-1}s^{-4/5}$ for $t\ge1$.
For the $k$th bump, choose the reference index $t_k:=\lceil e^{k^2}\rceil$ for $k\ge1$, and place its center at $x_{t_k}$. Thus $k$ labels the bumps, while $t_k$ selects a position from the auxiliary sequence $(x_t)$. These indices are prescribed before running Muon; they are not the actual iterations at which the algorithm reaches the bumps. Since $|x_t|\asymp t^{1/5}$ as $t\to\infty$, the centers satisfy $|x_{t_k}|\asymp e^{k^2/5}$ and become widely separated. Choose $k_0\ge2$ large enough that consecutive centers are separated by more than $2$ for $k\ge k_0$.

We now construct the gradient $H=F'$ by adding smooth bumps to a positive, integrable background. Let $\rho\in C_c^\infty((-1,1))$ satisfy $0\le\rho\le1$ and $\rho=1$ on $[-1/2,1/2]$, and define
\begin{equation}\label{eq:sparse-stoch-grad}
\begin{gathered}
 H(x):=\frac1{1+x^2}+\sum_{k\ge k_0}\frac1k\rho\bigl(k(x-x_{t_k})\bigr),\\
 F(x):=\int_{-\infty}^x H(u)\,du.
\end{gathered}
\end{equation}
The $k$th added bump has height $1/k$, is supported within distance $1/k$ of $x_{t_k}$, and equals $1/k$ throughout the central interval $[x_{t_k}-1/(2k),x_{t_k}+1/(2k)]$. Its area is $O(k^{-2})$, so the bump areas are summable and $F$ is well defined and lower bounded. The background term keeps $H$ strictly positive. Figure~\ref{fig:sparse-gradient-peaks} sketches $H$ and $F$.

For each actual Muon trajectory $(w_t)$ considered below, define $s_k:=\min\{t\ge2:w_t\le x_{t_k}+1/(2k)\}$. Thus $s_k$ is the first iteration at which the trajectory reaches or passes the right endpoint of the $k$th bump's central interval. Unlike the prescribed index $t_k$, this crossing time depends on the actual trajectory and may differ between the two regimes. The proof shows that, for all sufficiently large $k$, $s_k$ is finite and the crossing step lands inside the central interval, where $|F'(w_{s_k})|\ge1/k$.

\Needspace{14\baselineskip}
\begin{proposition}[No polynomial pointwise rate in either regime]
\label{prop:stoch-no-rate}
The objective $F$ in \eqref{eq:sparse-stoch-grad} is smooth and lower bounded. Let $L\ge1$ be a Lipschitz constant of $H=F'$. Initialize $w_1=M_0=0$ and use exact gradients. Consider either
\begin{enumerate}[label=\textup{(\roman*)},leftmargin=*,itemsep=2pt,topsep=3pt]
\item \emph{Fixed-step regime (Theorem~\ref{thm:fixed}):} any $\eps>0$, fixed $\beta\in[0,1)$ and $\theta\in[0,1]$, with $0<\eta\le\eps(1-\beta)/(4L\uJ)$;
\item \emph{Diminishing-step regime (Theorem~\ref{thm:stoch}):} any $\eps\ge0$, $\eta_t=\eta_0t^{-4/5}$, $\beta_t=1-\frac12t^{-1/2}$, and $\theta_t\in\{1,\beta_t\}$, with $\zeta_t=0$.
\end{enumerate}
In both regimes all assumptions of the respective theorem hold and $|F'(w_t)|\to0$. Nevertheless, for each trajectory and all sufficiently large $k$, the crossing times $s_k$ defined above satisfy
\[
 s_k\ge C e^{k^2/5},\quad
 |w_{s_k}-x_{t_k}|\le\frac1{2k},\quad |F'(w_{s_k})|\ge\frac1k,
\]
where $C>0$ is independent of $k$.
Consequently,
\begin{equation*}
 \limsup_{t\to\infty}t^\nu|F'(w_t)|=+\infty
 \quad\text{for every }\nu>0.
\end{equation*}
Both regimes can use the same positive $\eps$.
\end{proposition}

The rate obstruction comes from comparing the gradient size with the actual visit time $s_k$, rather than the auxiliary index $t_k$. For every $\nu>0$, the proposition gives $s_k^\nu|F'(w_{s_k})|\ge C^\nu e^{\nu k^2/5}/k\to\infty$. Thus the gradients still converge to zero, but no bound of the form $|F'(w_t)|=O(t^{-\nu})$ can hold at all sufficiently late iterations. Appendix~\ref{app:pointwise-proof} proves the proposition, including the fact that neither update skips the central intervals at late times.

\subsection{Density-One Rate Guarantees}
\label{subsec:density-rates}
The sparse peaks in Section~\ref{subsec:rate-separation} motivate a rate notion that captures the behavior of the overwhelming majority of iterations while remaining compatible with exceptional, very sparse peaks. We use the natural density of the iterations where a specified rate is exceeded.

\begin{definition}[Density-one rate]\label{def:density}
For a nonnegative sequence $h_t$ and a positive reference rate $\rho_t$, write $h_t=o_{\mathrm{dens}}(\rho_t)$ if
\begin{equation*}
 \forall c>0,\quad
 \lim_{T\to\infty}\frac{1}{T}
 \#\bigl\{1\le t\le T:h_t>c\,\rho_t\bigr\}=0.
\end{equation*}
\end{definition}
For stochastic trajectories the condition is understood almost surely, path by path. Unlike a pointwise $o(\rho_t)$ rate, this definition allows exceptional iterations whose asymptotic proportion vanishes.
The gradient-energy bounds in Theorems~\ref{thm:fixed} and~\ref{thm:stoch} give explicit rates in this sense.

\begin{theorem}[Density-one gradient rates for finite-step Muon]
\label{thm:density-rates}
Under the assumptions of Theorem~\ref{thm:fixed}, the gradient sequence satisfies, pathwise,
\begin{equation}\label{eq:dens-fixed}
 \norm{g_t}=o_{\mathrm{dens}}(t^{-1/2}).
\end{equation}
Under the assumptions of Theorem~\ref{thm:stoch} with $p<1$, the following rates hold almost surely:
\begin{equation}\label{eq:stoch-density-rates}
 \norm{g_t}=\begin{cases}
 o_{\mathrm{dens}}(t^{-(1-p)/2}),&\eps>0,\\
 o_{\mathrm{dens}}(t^{-(1-p)}),&\eps=0.
 \end{cases}
\end{equation}
\end{theorem}

To see the mechanism, the constant-step theorem yields $\sum_t\norm{g_t}^2<\infty$. Thus for any $c>0$,
\begin{equation*}
 \frac1T\#\left\{T\le t<2T:
 \norm{g_t}>\frac{c}{\sqrt T}\right\}
 \le\frac{1}{c^2}\sum_{t=T}^{2T-1}\norm{g_t}^2\longrightarrow0.
\end{equation*}
The same counting principle applies to the stochastic energy estimate~\eqref{eq:stoch-weighted-main} when $\eta_t\asymp t^{-p}$. For $\eps>0$, it yields the quadratic weighted summability $\sum_t\eta_t\norm{g_t}^2<\infty$; for $\eps=0$, the identity $\psi_0(x)=x$ gives the stronger linear bound $\sum_t\eta_t\norm{g_t}<\infty$. These two forms of energy control produce the different exponents in \eqref{eq:stoch-density-rates}. The general summability-to-density argument and the full proof of Theorem~\ref{thm:density-rates} appear in Appendix~\ref{app:proof-density-theorem}.

The admissible schedules also yield explicit limiting exponents. As $p\downarrow3/4$, for any sufficiently small $\delta>0$ one can choose parameters so that
\begin{equation*}
 \norm{g_t}=o_{\mathrm{dens}}(t^{-1/8+\delta})
 \quad(\eps>0),\quad
 \norm{g_t}=o_{\mathrm{dens}}(t^{-1/4+\delta})
 \quad(\eps=0),\quad\as.
\end{equation*}
For example, the balanced schedule $q=2p/3$ from \eqref{eq:balanced-schedule} realizes these bounds by taking $p=3/4+2\delta$ for $\eps>0$ or $p=3/4+\delta$ for $\eps=0$.

Density-one rates quantify gradient decay at asymptotically almost all late iterations of a single run, rather than identifying only a best-performing iterate or bounding an average. This gives a useful trajectory-level interpretation when optimization is inspected over many training steps: the exceptional iterations occupy a vanishing fraction of the run.

%% file: f_and_h_tikz.tex
%
%
%
%

\begingroup
\definecolor{Fcurve}{HTML}{156084}
\definecolor{Hcurve}{HTML}{BD522E}
\definecolor{Fbackground}{HTML}{8A939E}
\pgfmathsetlengthmacro{\FHpanelwidth}{0.345*\linewidth}
\pgfmathsetlengthmacro{\FHpanelgap}{0.14*\linewidth}
\pgfmathsetlengthmacro{\FHpanelheight}{0.31*\linewidth}
\pgfmathsetlengthmacro{\FHinsetwidth}{0.21*\linewidth}
\pgfmathsetlengthmacro{\FHinsetheight}{0.115*\linewidth}
\pgfplotsset{
  compat=1.18,
  FH main/.style={
    scale only axis,
    width=\FHpanelwidth,height=\FHpanelheight,
    xmin=-130,xmax=-3,
    xtick={-120,-90,-60,-30},
    xlabel={$x$},
    axis lines=left,
    axis line style={black!60},
    tick align=outside,tick style={black!60},
    tick label style={font=\small},
    label style={font=\normalsize},
    title style={font=\normalsize,yshift=2pt},
    ymajorgrids=true,grid style={black!9,line width=0.3pt},
    enlargelimits=false,
    clip=true,
  }
}


\begin{tikzpicture}
  \begin{axis}[
    FH main,name=Fplot,
    ymin=0,ymax=1.05,ytick={0,0.2,0.4,0.6,0.8,1},
    ylabel={$F(x)$},title={(a) Objective: left tail},
  ]
    \addplot[Fcurve,line width=1.15pt,no marks]
      table[x=x,y=F] {f_h_data.dat};
    \fill[Fcurve] (axis cs:-6.6861875006,0.6504684948) circle[radius=1.25pt];
    \node[anchor=north east,yshift=-4pt,text=Fcurve,font=\small] at (axis cs:-6.6861875006,0.6504684948) {$x_{t_2}$};
    \fill[Fcurve] (axis cs:-25.8110095366,0.3664006175) circle[radius=1.25pt];
    \node[below=5pt,text=Fcurve,font=\small] at (axis cs:-25.8110095366,0.3664006175) {$x_{t_3}$};
    \fill[Fcurve] (axis cs:-118.2251125130,0.2523225435) circle[radius=1.25pt];
    \node[below=4pt,text=Fcurve,font=\small] at (axis cs:-118.2251125130,0.2523225435) {$x_{t_4}$};
    \coordinate (Finset) at (rel axis cs:0.16,0.53);
  \end{axis}

  \begin{axis}[
    FH main,name=Hplot,
    at={(Fplot.east)},anchor=west,xshift=\FHpanelgap,
    ymin=-0.015,ymax=0.67,ytick={0,0.1,0.2,0.3,0.4,0.5,0.6},
    ylabel={$F'(x)=H(x)$},title={(b) Derivative: sparse peaks},
    legend style={draw=none,fill=none,font=\footnotesize,
      at={(0.02,0.80)},anchor=north west,cells={anchor=west}},
  ]
    \addplot[Hcurve,line width=1.05pt,no marks,forget plot]
      table[x=x,y=H] {f_h_data.dat};
    \addplot[Fbackground,dashed,line width=0.8pt,no marks]
      table[x=x,y=B] {f_h_data.dat};
    \addlegendentry{Baseline $1/(1+x^2)$}
    \draw[->,black!70,line width=0.7pt]
      (axis cs:-77,0.605) -- (axis cs:-111,0.605);
    \node[anchor=west,font=\footnotesize,text=black!75]
      at (axis cs:-76,0.605) {Iterates move left};
    \node[anchor=south east,yshift=3pt,text=Hcurve,font=\small] at (axis cs:-6.6861875006,0.5218793948) {$k=2$};
    \node[above=3pt,text=Hcurve,font=\small] at (axis cs:-25.8110095366,0.3348321158) {$k=3$};
    \node[above=3pt,text=Hcurve,font=\small] at (axis cs:-118.2251125130,0.2500715401) {$k=4$};
  \end{axis}

  \begin{axis}[
    at={(Finset)},anchor=south west,
    scale only axis,width=\FHinsetwidth,height=\FHinsetheight,
    xmin=-130,xmax=7,ymin=0,ymax=3.8,
    xtick={-120,0},ytick={0,3},
    axis lines=left,axis line style={black!55},
    tick align=outside,tick label style={font=\scriptsize},
    title={Wider view of $F$},title style={font=\footnotesize,yshift=-2pt},
    axis background/.style={fill=white},
    enlargelimits=false,clip=true,
  ]
    \path[fill=Fcurve!8,draw=none]
      (axis cs:-130,0) rectangle (axis cs:-3,3.8);
    \addplot[Fcurve,line width=0.8pt,no marks]
      table[x=x,y=F] {f_h_data.dat};
  \end{axis}
\end{tikzpicture}
\endgroup

%% file: appendices/a-preliminaries.tex
\section{Supporting Results for Section~\ref{sec:setup}: Preliminaries}\label{app:preliminaries}
This appendix establishes the finite Newton--Schulz map estimates and the objective-gradient bound used throughout, then records the alternative input normalization.

\subsection{Proof of Lemma~\ref{lem:map}}
\label{app:proof-lem-map}
\begin{proof}
First consider the scalar polynomial. For every real $s\ne0$,
\begin{equation}\label{eq:polynomialpositive}
 \frac{P(s)}s=a+bs^2+cs^4
 =c\left(s^2+\frac{b}{2c}\right)^2+a-\frac{b^2}{4c}
 \ge a-\frac{b^2}{4c}>0.
\end{equation}
The final strict inequality follows by direct substitution of \eqref{eq:coeff} (the difference is approximately $0.6386$). In particular, $P(s)>0$ whenever $s>0$; iteration shows $\varphi_J(s)>0$ for $s>0$. Since $P(s)=as+O(s^3)$ as $s\downarrow0$, the chain rule, or induction over $J$, gives $\varphi_J(s)=a^Js+O(s^3)$. Thus $q_J$ extends continuously and strictly positively to the compact interval $[0,1]$, proving the claim concerning $\ellJ$ and $\uJ$.

If $N=0$, then $\Phi_\eps(N)=0$ and every assertion is immediate. Otherwise choose a thin singular-value decomposition $N=U\diag(\sigma_1,\ldots,\sigma_r)V^\top$, with $\sigma_i\ge0$ and $r=\min\{m,n\}$; zero singular values are allowed. Direct multiplication gives
\begin{align*}
 Z_j&=U\diag(z_{1,j},\ldots,z_{r,j})V^\top,\\
 z_{i,0}&=\frac{\sigma_i}{R+\eps},\quad
 z_{i,j+1}=P(z_{i,j}).
\end{align*}
Because $\sigma_i\le R$, all the inputs $z_{i,0}$ lie in $[0,1]$. The coefficients in \eqref{eq:polynomialpositive} ensure that $z_{i,j}\ge0$ for all $j$. Writing $d_i=z_{i,J}$, we therefore have
\begin{equation}\label{eq:di}
  \frac{\ellJ\sigma_i}{R+\eps}\le d_i
  \le\frac{\uJ\sigma_i}{R+\eps}\quad(1\le i\le r).
\end{equation}
Squaring \eqref{eq:di}, summing over $i$, and using $\sum_i\sigma_i^2=R^2$ proves \eqref{eq:radial}; summing $\sigma_i d_i$ proves \eqref{eq:align}. Finally, $\sigma_i\ge(R+\eps)d_i/\uJ$, whence
\[
  \sigma_i d_i\ge\frac{R+\eps}{\uJ}d_i^2
  \ge\frac{\eps}{\uJ}d_i^2.
\]
Summing yields \eqref{eq:sector}.
\end{proof}

\subsection{Gradient self-bound}
\label{app:proof-lem-self-bound}
\begin{lemma}[Gradient self-bound]\label{lem:self-bound}
Under Assumption~\ref{ass:smooth}, for all $X\in\RR^{m\times n}$,
\begin{equation}\label{eq:selfbound}
  \norm{\nabla F(X)}^2\le2L\bigl(F(X)-\Finf\bigr).
\end{equation}
\end{lemma}

\begin{proof}
The descent lemma applied at $Y=X-L^{-1}\nabla F(X)$ gives $F(Y)\le F(X)-\norm{\nabla F(X)}^2/(2L)$. Combining this with $F(Y)\ge\Finf$ and rearranging proves the result.
\end{proof}

\subsection{Maximum-Type Input Normalization}
\label{app:variants}
We record a normalization variant within the same exact-arithmetic analysis.

\begin{proposition}[Maximum-type input floor]\label{prop:maxnorm}
Replace \eqref{eq:ns}'s initial normalization by
\[
 Z_0=N/\max\{\frob N,\eps\},\quad\eps>0,
\]
and denote the corresponding finite-NS map by $\Phi_\eps^{\max}$. Then
\begin{align}
 \frob{\Phi_\eps^{\max}(N)}
 &\le\uJ,\quad
 \frob{\Phi_\eps^{\max}(N)}
 \ge\ellJ\frac{\frob N}{\frob N+\eps},\notag\\
 \ip{N}{\Phi_\eps^{\max}(N)}_{\mathrm F}
 &\ge\frac{\eps}{\uJ}\frob{\Phi_\eps^{\max}(N)}^2.\label{eq:maxsector}
\end{align}
Accordingly Theorem~\ref{thm:fixed} holds, with the same sufficient condition \eqref{eq:fixed-stepsize}, for this modified normalization. The diminishing-step proofs carry over as well.
\end{proposition}
\begin{proof}
Set $\tau=\max\{\frob N,\eps\}\ge\eps$. For singular values $\sigma_i$ of $N$, one has $\sigma_i/\tau\in[0,1]$ and the finite polynomial bounds give $\ellJ\sigma_i/\tau\le d_i\le\uJ\sigma_i/\tau$. Thus $\norm{(d_i)}_2\le\uJ\frob N/\tau\le\uJ$ and $\norm{(d_i)}_2\ge\ellJ\frob N/\tau \ge\ellJ\frob N/(\frob N+\eps)$. The inequality $\sigma_id_i\ge\tau d_i^2/\uJ\ge \eps d_i^2/\uJ$ proves \eqref{eq:maxsector}. These reproduce all the estimates invoked in the fixed-step argument. The alignment inequality used for diminishing steps also follows from $\sum_i\sigma_id_i\ge\ellJ\frob N^2/\tau \ge\ellJ\frob N^2/(\frob N+\eps)$.
\end{proof}

%% file: appendices/b-fixed-step-proofs.tex
\section{Proofs for Section~\ref{sec:fixed}: Fixed-Step Stationarity}\label{app:fixed-proofs}
We first establish the unstabilized two-cycle, then give the descent--tracking Lyapunov proof of Theorem~\ref{thm:fixed}. We finish with the square-summable-error extension and scalar local stability.

\subsection{Proof of the Unstabilized Two-Cycle}
\label{app:cycle-proof}
The quadratic two-cycle in Proposition~\ref{prop:cycle} holds for all fixed momentum and pre-NS mixing coefficients in the stated ranges. Its explicit scalar recursion is as follows.

\begin{proof}[Proof of Proposition~\ref{prop:cycle}]
For every nonzero scalar $n$, zero-stabilizer normalization gives $Z_0=\sgn(n)$; the scalar recurrence is $\varphi_J$ and hence
\begin{equation}\label{eq:scalar-sign}
   \Phi_0(n)=c_J\sgn(n).
\end{equation}
If $w_s=(-1)^{s-1}\delta/2$ for $1\le s\le t$, then unrolling $M_t=\beta M_{t-1}+(1-\beta)w_t$ yields
\begin{align*}
 M_t&=(1-\beta)\frac{\delta}{2}
   \sum_{s=1}^t\beta^{t-s}(-1)^{s-1}\\
 &=(-1)^{t-1}\frac{(1-\beta)\delta}{2(1+\beta)}
      \bigl[1-(-\beta)^t\bigr].
\end{align*}
For $\beta\in[0,1)$ the bracket is strictly positive, so $M_t$ and $g_t=w_t$ have the same nonzero sign. Their convex mixture $N_t=\theta M_t+(1-\theta)g_t$ has that sign as well. By \eqref{eq:scalar-sign}, $w_{t+1}=w_t-\eta c_J\sgn(w_t)=-w_t$. This closes the induction beginning at $w_1=\delta/2$.
\end{proof}

\subsection{Descent--Tracking Lyapunov Proof of Theorem~\ref{thm:fixed}}
\subsubsection{Proof of Lemma~\ref{lem:fixed-descent}}
\label{app:proof-lem-fixed-descent}
Use the fixed parameters and constants from Section~\ref{sec:fixed}, under Assumption~\ref{ass:smooth}.
\begin{proof}
The smoothness inequality and update~\eqref{eq:update} give
\[
 F(W_{t+1})-F(W_t)
 \le-\eta\ip{g_t}{D_t}+\frac{L\eta^2}{2}\norm{D_t}^2.
\]
Writing $g_t=N_t-\theta e_t$, applying~\eqref{eq:sector}, and using~\eqref{eq:constants} yield
\begin{equation}\label{eq:fixed-descent0}
 F(W_{t+1})-F(W_t)
 \le-\eta\kappa\norm{D_t}^2
 +\eta\ip{\theta e_t}{D_t}
 +\frac{L\eta^2}{2}\norm{D_t}^2.
\end{equation}
Young's inequality gives $\ip r d\le\kappa\norm d^2/2+\norm r^2/(2\kappa)$. Since $\norm{\theta e_t}^2\le\norm{e_t}^2$, substituting this estimate into~\eqref{eq:fixed-descent0} proves \eqref{eq:fixed-descent}.
\end{proof}

\subsubsection{Proof of Lemma~\ref{lem:tracking}}
\label{app:proof-lem-tracking}
\begin{proof}
The recursion~\eqref{eq:ema}, evaluated at time $t+1$, gives
\begin{equation}\label{eq:erecur}
 e_{t+1}=\beta(e_t+g_t-g_{t+1}).
\end{equation}
Write $\Delta_t=g_t-g_{t+1}$. Since $0\le\beta<1$,
\begin{equation}\label{eq:tracking-jensen}
 \norm{e_{t+1}}^2
 =\beta^2\norm{e_t+\Delta_t}^2
 \le\beta\norm{e_t+\Delta_t}^2.
\end{equation}
For $0<\beta<1$, Young's squared-norm inequality with $\rho=\alpha/(2\beta)$ gives
\begin{align*}
 \beta\norm{e_t+\Delta_t}^2
 &\le\beta(1+\rho)\norm{e_t}^2
   +\beta(1+\rho^{-1})\norm{\Delta_t}^2\\
 &=\frac{1+\beta}{2}\norm{e_t}^2
   +\frac{\beta(1+\beta)}{\alpha}\norm{\Delta_t}^2.
\end{align*}
By Assumption~\ref{ass:smooth} and~\eqref{eq:update}, $\norm{\Delta_t}\le L\eta\norm{D_t}$; substituting into \eqref{eq:tracking-jensen} proves~\eqref{eq:tracking}. For $\beta=0$, \eqref{eq:erecur} reduces to $e_{t+1}=0$, and the same estimate holds.
\end{proof}

\subsubsection{Quantitative Lyapunov decrease}
\label{app:proof-lem-lyap}
\begin{lemma}[Quantitative Lyapunov decrease]\label{lem:lyap}
Define
\begin{equation*}
 V_t=F(W_t)+\frac{2\eta}{\kappa\alpha}\norm{e_t}^2.
\end{equation*}
If \eqref{eq:fixed-stepsize} holds, then for all $t\ge1$,
\begin{equation}\label{eq:lyap}
 V_{t+1}-V_t
 \le-\frac{\eta\kappa}{4}\norm{D_t}^2
     -\frac{\eta}{2\kappa}\norm{e_t}^2.
\end{equation}
\end{lemma}

\begin{proof}
Multiply \eqref{eq:tracking} by $C=2\eta/(\kappa\alpha)$ and add \eqref{eq:fixed-descent}. The coefficient of $\norm{e_t}^2$ is
\[
 \frac{\eta}{2\kappa}
  +C\left(\frac{1+\beta}{2}-1\right)
  =\frac{\eta}{2\kappa}-\frac{\eta}{\kappa}
  =-\frac{\eta}{2\kappa}.
\]
The remaining negative coefficient of $\norm{D_t}^2$ is
\begin{equation*}
 Q=\frac{\eta\kappa}{2}-\frac{L\eta^2}{2}
   -\frac{2\beta(1+\beta)L^2\eta^3}{\kappa\alpha^2}.
\end{equation*}
To verify that it is positive under the explicit step-size condition, put $x=L\eta/(\kappa\alpha)$, so $0<x\le1/4$. Then
\begin{align*}
\frac{Q}{\eta\kappa}
 &=\frac12-\frac{\alpha x}{2}-2\beta(1+\beta)x^2\\
 &\ge\frac12-\frac{\alpha}{8}-\frac{\beta(1+\beta)}8
  =\frac{3-\beta^2}{8}\ge\frac14.
\end{align*}
This lower bound holds also at $\beta=0$. Substituting $Q\ge\eta\kappa/4$ into the combined inequality gives \eqref{eq:lyap}.
\end{proof}

\subsubsection{Proof of Theorem~\ref{thm:fixed}}
\label{app:proof-thm-fixed}
\begin{proof}
\emph{Part I: Lyapunov summability.} For exact gradients, \eqref{eq:lyap} and the lower bound \(V_{T+1}\ge F(W_{T+1})\ge\Finf\) give
\begin{equation}\label{eq:telescoped}
 \frac{\eta\kappa}{4}\sum_{t=1}^{T}\norm{D_t}^2
 +\frac{\eta}{2\kappa}\sum_{t=1}^{T}\norm{e_t}^2
 \le V_1-\Finf=:H<\infty.
\end{equation}
Thus the first two series in \eqref{eq:fixed-result} converge. Moreover, \eqref{eq:lyap} ensures \(V_t\le V_1\), hence
\begin{equation}\label{eq:Ftbound}
 0\le F(W_t)-\Finf\le V_t-\Finf\le H\quad(t\ge1).
\end{equation}

\emph{Part II: square-summability of gradients.} Lemma~\ref{lem:self-bound} and \eqref{eq:Ftbound} give
\begin{equation*}
 \norm{g_t}\le\sqrt{2LH}\quad(t\ge1).
\end{equation*}
By \eqref{eq:telescoped}, also \(\norm{e_t}\le\sqrt{2\kappa H/\eta}\) for every $t$. Since \(N_t=g_t+\theta e_t\), the inputs are uniformly bounded:
\begin{equation*}
 H_N:=\sqrt{2LH}+\sqrt{2\kappa H/\eta},
 \quad\norm{N_t}\le H_N\quad(t\ge1).
\end{equation*}
The radial estimate \eqref{eq:radial} therefore yields
\begin{equation}\label{eq:NvsD}
 \norm{N_t}^2
 \le\frac{(H_N+\eps)^2}{\ellJ^2}\norm{D_t}^2.
\end{equation}
In particular, \(\sum_t\norm{N_t}^2<\infty\). Since \(g_t=N_t-\theta e_t\),
\begin{equation*}
 \norm{g_t}^2\le2\norm{N_t}^2+2\norm{e_t}^2.
\end{equation*}
The third summability assertion in \eqref{eq:fixed-result} now follows. Every square-summable sequence converges to zero; thus \(\norm{g_t}\to0\), as claimed.
\end{proof}

\subsection{Fixed-step extension to square-summable gradient perturbations}
\label{app:fixed-perturbations}
The exact-gradient restriction in Theorem~\ref{thm:fixed} may be relaxed to perturbations of finite total squared energy. We keep this extension separate because the main theorem and its proof require no stochastic assumptions.

\begin{theorem}[Fixed-step stationarity with square-summable errors]
\label{thm:fixed-perturbed}
Suppose Assumption~\ref{ass:smooth} holds. Use \(G_t=g_t+\zeta_t\), fix finite \(J\ge1\), \(\eps>0\), \(\beta\in[0,1)\), and \(\theta\in[0,1]\), and choose the constant step size according to \eqref{eq:fixed-stepsize}. If the perturbations satisfy
\begin{equation}\label{eq:pathwise-noise}
 \sum_{t=1}^{\infty}\norm{\zeta_t}^2<\infty,
\end{equation}
then, pathwise,
\begin{equation}\label{eq:perturbed-result}
 \sum_{t=1}^{\infty}\norm{D_t}^2<\infty,\quad
 \sum_{t=1}^{\infty}\norm{M_t-g_t}^2<\infty,\quad
 \sum_{t=1}^{\infty}\norm{g_t}^2<\infty.
\end{equation}
In particular, \(\norm{g_t}\to0\). If the perturbations are random and satisfy
\begin{equation}\label{eq:integrated-noise}
 \sum_{t=1}^{\infty}\EE\norm{\zeta_t}^2<\infty,
\end{equation}
then \eqref{eq:perturbed-result} holds almost surely and
\begin{equation}\label{eq:fixed-MS}
 \EE\norm{\nabla F(W_t)}^2\longrightarrow0.
\end{equation}
Neither conditional unbiasedness nor independence is required.
\end{theorem}

\begin{proof}
Set \(\alpha=1-\beta\), \(\kappa=\eps/\uJ\), and \(e_t=M_t-g_t\). Now \(N_t-g_t=\theta e_t+(1-\theta)\zeta_t\). The sector inequality \eqref{eq:sector}, smoothness, Young's inequality, and \(\norm{\theta e_t+(1-\theta)\zeta_t}^2 \le\norm{e_t}^2+\norm{\zeta_t}^2\) give
\begin{equation}\label{eq:perturbed-descent}
 F(W_{t+1})-F(W_t)
 \le-\left(\frac{\eta\kappa}{2}-\frac{L\eta^2}{2}\right)
      \norm{D_t}^2
    +\frac{\eta}{2\kappa}
      \bigl(\norm{e_t}^2+\norm{\zeta_t}^2\bigr).
\end{equation}
The exact tracking recursion is
\begin{equation*}
 e_{t+1}=\beta(e_t+g_t-g_{t+1})+\alpha\zeta_{t+1}.
\end{equation*}
By convexity of the squared norm and Young's inequality, for \(0<\beta<1\),
\begin{equation}\label{eq:perturbed-tracking}
 \norm{e_{t+1}}^2
 \le\frac{1+\beta}{2}\norm{e_t}^2
    +\frac{\beta(1+\beta)L^2\eta^2}{\alpha}\norm{D_t}^2
    +\alpha\norm{\zeta_{t+1}}^2.
\end{equation}
For \(\beta=0\), this bound follows directly from \(e_{t+1}=\zeta_{t+1}\).

With \(V_t=F(W_t)+2\eta\norm{e_t}^2/(\kappa\alpha)\), the same coefficient calculation as in Appendix~\ref{app:proof-lem-lyap} (which uses only the step-size restriction) combines \eqref{eq:perturbed-descent} and \eqref{eq:perturbed-tracking} into
\begin{align}\label{eq:perturbed-lyap}
 V_{t+1}-V_t\le{}&
 -\frac{\eta\kappa}{4}\norm{D_t}^2
 -\frac{\eta}{2\kappa}\norm{e_t}^2\\
 &+\frac{\eta}{2\kappa}\norm{\zeta_t}^2
 +\frac{2\eta}{\kappa}\norm{\zeta_{t+1}}^2.\notag
\end{align}
Fix any trajectory satisfying \eqref{eq:pathwise-noise} and put \(Z:=\sum_{t\ge1}\norm{\zeta_t}^2\) and \(H:=V_1-\Finf+5\eta Z/(2\kappa)<\infty\). Telescoping \eqref{eq:perturbed-lyap} and using \(V_{T+1}\ge\Finf\) show that
\begin{equation*}
 \frac{\eta\kappa}{4}\sum_{t=1}^{T}\norm{D_t}^2
 +\frac{\eta}{2\kappa}\sum_{t=1}^{T}\norm{e_t}^2
 \le H\quad(T\ge1).
\end{equation*}
The same inequality implies \(0\le F(W_t)-\Finf\le H\) for all $t$. By Lemma~\ref{lem:self-bound}, \(\norm{g_t}\le\sqrt{2LH}\). Also \(\norm{e_t}\le\sqrt{2\kappa H/\eta}\) and \(\norm{\zeta_t}\le\sqrt Z\). Consequently,
\[
 \norm{N_t}\le H_N
 :=\sqrt{2LH}+\sqrt{2\kappa H/\eta}+\sqrt Z<\infty.
\]
Using \eqref{eq:radial} as in \eqref{eq:NvsD} gives \(\sum_t\norm{N_t}^2<\infty\). Finally,
\[
 \norm{g_t}^2\le
 3\norm{N_t}^2+3\norm{e_t}^2+3\norm{\zeta_t}^2,
\]
so the third series in \eqref{eq:perturbed-result} converges.

Under \eqref{eq:integrated-noise}, Tonelli's theorem gives \(Z<\infty\) almost surely, hence the pathwise conclusions hold almost surely. Deterministic finite initialization and \(\EE\norm{\zeta_1}^2<\infty\) ensure \(\EE V_1<\infty\), so \(\EE H<\infty\). For every $t$, Lemma~\ref{lem:self-bound} gives \(0\le\norm{g_t}^2\le2LH\). Dominated convergence proves \eqref{eq:fixed-MS}.
\end{proof}

\subsubsection{Growing Batches with Fixed Hyperparameters}
\begin{corollary}[Growing batches with constant hyperparameters]
\label{cor:growingbatch}
Under the assumptions of Theorem~\ref{thm:fixed-perturbed}, if \(\EE\norm{\zeta_t}^2\le\sigma^2/b_t\) and \(\sum_t b_t^{-1}<\infty\), then full-sequence stationarity holds almost surely and in mean square. A sufficient choice is \(b_t=\lceil t^{1+\delta}\rceil\) for \(\delta>0\).
\end{corollary}
\begin{proof}
The moment and batch conditions imply \eqref{eq:integrated-noise}. Apply Theorem~\ref{thm:fixed-perturbed}.
\end{proof}

\subsection{Exact Local Stability Region on Quadratics}\label{app:obstructions}
We characterize the precise Schur stability region of the scalar finite-NS dynamics. This criterion describes local behavior around the stationary point and complements the global Lyapunov condition for smooth lower-bounded objectives.

\begin{proposition}[Strict Schur stability of the scalar linearization]
\label{prop:schur}
Let $F(w)=\lambda w^2/2$ with $\lambda>0$ and use exact gradients, fixed $\eps>0$, $\eta>0$, $\beta\in[0,1)$, and $\theta\in[0,1]$. The linearized state dynamics at the origin are strictly Schur stable if and only if
\begin{equation}\label{eq:stability-region}
 0<s:=\frac{\eta\lambda a^J}{\eps}
 <\frac{2(1+\beta)}{1+\beta-2\theta\beta}.
\end{equation}
Within this open range the nonlinear origin is locally asymptotically stable; above the upper endpoint it is unstable. The equality case is not classified by this linearization.
\end{proposition}

For ordinary EMA momentum ($\theta=1$), the upper bound is $2(1+\beta)/(1-\beta)$. For Nesterov-style mixing ($\theta=\beta$), it becomes $2(1+\beta)/\bigl((1-\beta)(1+2\beta)\bigr)$.

\begin{proof}
Near $n=0$, $n/(|n|+\eps)=n/\eps+o(|n|)$ and $\varphi_J(z)=a^Jz+o(|z|)$. Thus
\begin{equation}\label{eq:scalar-derivative}
 \Phi_\eps(n)=kn+o(|n|),\quad k=\frac{a^J}{\eps}.
\end{equation}
Take the state variable $(w_t,M_{t-1})$. Eliminating the current momentum gives
\[
 M_t=\beta M_{t-1}+(1-\beta)\lambda w_t,
 \quad N_t=\theta\beta M_{t-1}+(1-\theta\beta)\lambda w_t.
\]
Using \eqref{eq:scalar-derivative}, the Jacobian of the map $(w_t,M_{t-1})\mapsto(w_{t+1},M_t)$ at $(0,0)$ is
\begin{equation*}
 A=\begin{pmatrix}
  1-s(1-\theta\beta)&-\eta k\theta\beta\\
  (1-\beta)\lambda&\beta
 \end{pmatrix}.
\end{equation*}
Its trace and determinant are, by direct calculation,
\begin{equation}\label{eq:trace-det}
 \tau=1+\beta-s(1-\theta\beta),
 \quad d=\beta\bigl[1-(1-\theta)s\bigr].
\end{equation}
For the characteristic polynomial $z^2-\tau z+d$, the strict real quadratic Schur (Jury) conditions are $1-\tau+d>0$, $1+\tau+d>0$, and $1-d>0$. Substituting \eqref{eq:trace-det} gives exactly
\begin{align}
 1-\tau+d&=s(1-\beta),\label{eq:schur1}\\
 1+\tau+d&=2(1+\beta)-s(1+\beta-2\theta\beta),\label{eq:schur2}\\
 1-d&=1-\beta+\beta(1-\theta)s.\label{eq:schur3}
\end{align}
Because $s>0$, $\beta<1$, and $0\le\theta\le1$, \eqref{eq:schur1} and \eqref{eq:schur3} are strictly positive. Furthermore $1+\beta-2\theta\beta=1-\beta+2\beta(1-\theta)>0$. Consequently \eqref{eq:schur2} is equivalent to the upper bound in \eqref{eq:stability-region}. The state map is continuously differentiable in a neighborhood of the origin because $\eps>0$; the discrete-time linearization theorem gives local asymptotic stability when $A$ is strictly Schur stable. If $s$ is above the threshold, the characteristic polynomial satisfies $p(-1)=1+\tau+d<0$ and $p(z)\to+\infty$ as $z\to-\infty$. Hence an eigenvalue is less than $-1$, so the origin is unstable. At equality $-1$ is an eigenvalue, and linearization alone does not settle nonlinear stability.
\end{proof}

\paragraph{Global and local stability certificates.}
The global Lyapunov condition and the exact scalar Schur criterion address complementary objectives. The former guarantees full-sequence stationarity from arbitrary initialization under the smoothness and exact-gradient assumptions; the latter gives a precise local threshold around a quadratic stationary point. For $F(w)=\lambda w^2/2$, $\lambda>0$, exact gradients, and Nesterov-style mixing $\theta=\beta$, Proposition~\ref{prop:schur} yields local asymptotic stability whenever
\begin{equation*}
  0<\eta<\eta_{\mathrm{loc}}
  :=\frac{2(1+\beta)\eps}
   {\lambda a^J(1-\beta)(1+2\beta)}.
\end{equation*}
The global sufficient condition from Theorem~\ref{thm:fixed}, specialized to $L=\lambda$, is $\eta\le\eta_{\mathrm{glob}}:=\eps(1-\beta)/(4\lambda\uJ)$. These distinct certificates satisfy
\begin{equation*}
 \frac{\eta_{\mathrm{loc}}}{\eta_{\mathrm{glob}}}
 =\frac{8\uJ}{a^J}\,
   \frac{1+\beta}{(1-\beta)^2(1+2\beta)}
 \ge\frac{8(1+\beta)}{(1-\beta)^2(1+2\beta)},
\end{equation*}
because $\uJ\ge q_J(0)=a^J$. Their dependence on $\beta$ also differs: the global sufficient threshold scales as $1-\beta$ while the scalar local threshold scales as $(1-\beta)^{-1}$ as $\beta\uparrow1$. These are separate, mathematically precise characterizations of global dissipation and local linear stability.

%% file: appendices/c-stochastic-proofs.tex
\section{Proofs for Section~\ref{sec:stoch}: Stochastic Stationarity}\label{app:stochastic-proofs}
We prove the persistent-noise obstruction, develop the diminishing-step stationarity criterion, and then derive the noise-filtering estimates for Theorem~\ref{thm:stoch}. Additional fixed-momentum regimes are recorded at the end.

\subsection{Proof of the Persistent-Noise Floor}
\label{app:noise-proof}
We prove the probability and expectation lower bounds of Proposition~\ref{prop:noise} under bounded symmetric unbiased perturbations.

\begin{proof}[Proof of Proposition~\ref{prop:noise}]
Using the current EMA update in the mixture yields
\begin{equation*}
 N_t=\theta\beta M_{t-1}
 +(1-\theta\beta)(w_t+\zeta_t).
\end{equation*}
Let $a_0=(1-\theta\beta)\sigma>0$. Conditional on the past, the two possible values of $N_t$ differ by $2a_0$. If two real numbers have distance $2a_0$, at least one has absolute value at least $a_0$. Since the two possibilities are equiprobable,
\begin{equation}\label{eq:noiseprob}
 \Pr\bigl(|N_t|\ge a_0\mid\calF_{t-1}\bigr)\ge\tfrac12.
\end{equation}
By \eqref{eq:radial}, whenever $|N_t|\ge a_0$, $|D_t|\ge\ellJ a_0/(a_0+\eps)=:d_0>0$. It follows that $\EE|D_t|\ge d_0/2$ for every $t$. But $w_{t+1}-w_t=-\eta D_t$, so by the triangle inequality
\begin{equation}\label{eq:lower-w}
  \eta d_0/2\le\eta\EE|D_t|
      \le\EE|w_{t+1}|+\EE|w_t|.
\end{equation}
All finite-time expectations exist, because $|D_t|\le\uJ$ ensures $|w_t|\le|w_1|+(t-1)\eta\uJ$. Taking upper limits in \eqref{eq:lower-w}, allowing an infinite upper limit, yields $\limsup_t\EE|w_t|\ge\eta d_0/4$, and hence $\limsup_t\EE|F'(w_t)|\ge\eta d_0/4$, which is exactly \eqref{eq:noise-floor}.

For the probability claim, \eqref{eq:noiseprob} and the radial lower bound imply $\Pr(|w_{t+1}-w_t|\ge\eta d_0)\ge1/2$. With $h_0=\eta d_0/2$, the triangle inequality then gives
\[
 \Pr(|w_t|\ge h_0)+\Pr(|w_{t+1}|\ge h_0)
 \ge\tfrac12\quad(t\ge1).
\]
Consequently, at least one member of every adjacent pair has probability at least $1/4$. Since $|F'(w_t)|=|w_t|$, this proves \eqref{eq:bad-probability}.
\end{proof}

\subsection{General Diminishing-Step Stationarity Criterion}\label{app:aux-regimes}
For diminishing steps, use the standard summability conditions
\begin{equation}\label{eq:steps}
 \sum_{t=1}^{\infty}\eta_t=\infty,\quad
 \sum_{t=1}^{\infty}\eta_t^2<\infty.
\end{equation}

\begin{theorem}[Diminishing-step stationarity criterion]\label{thm:criterion}
Suppose Assumption~\ref{ass:smooth} holds and \(\eps\ge0\) is fixed. Let the deterministic positive step sizes satisfy \eqref{eq:steps}. If along a deterministic trajectory
\begin{equation}\label{eq:weighted-error-path}
 \sum_{t=1}^{\infty}\eta_t\norm{r_t}<\infty,
\end{equation}
then \(\norm{g_t}\to0\). For random iterates, if
\begin{equation}\label{eq:weighted-error-mean}
 \sum_{t=1}^{\infty}\eta_t\EE\norm{r_t}<\infty,
\end{equation}
then \(\norm{g_t}\to0\) almost surely and \(\EE\norm{g_t}^2\to0\).
\end{theorem}

\subsubsection{Inexact-input descent bound}
\label{app:proof-lem-inexact}
\begin{lemma}[Inexact-input descent]\label{lem:inexact}
Under Assumption~\ref{ass:smooth}, for any $\eta_t>0$ and $\eps\ge0$, the update~\eqref{eq:update} satisfies
\begin{align}
 \ip{g_t}{D_t}
 &\ge\ellJ\psi_\eps(\norm{g_t})-(\ellJ+B)\norm{r_t},
   \label{eq:inexact-align}\\
 F(W_{t+1})
 &\le F(W_t)-\ellJ\eta_t\psi_\eps(\norm{g_t})
   +(\ellJ+B)\eta_t\norm{r_t}
   +\frac{LB^2}{2}\eta_t^2.\label{eq:inexact-dec}
\end{align}
Moreover, the gradient norm cannot change too quickly:
\begin{equation}\label{eq:stoch-regularity-explain}
 \bigl|\norm{g_{t+1}}-\norm{g_t}\bigr|\le LB\eta_t.
\end{equation}
\end{lemma}

\begin{proof}
The derivative of $\psi_\eps$ on $(0,\infty)$ is
\[
 \psi_\eps'(x)
 =\frac{x(x+2\eps)}{(x+\eps)^2}
 =1-\frac{\eps^2}{(x+\eps)^2}\in[0,1].
\]
For $\eps=0$, $\psi_0(x)=x$. In either case the extension to $[0,\infty)$ is nonnegative, nondecreasing, and $1$-Lipschitz. Applying \eqref{eq:align} directly to the single matrix gives
\[
 \ip{N_t}{D_t}\ge\ellJ\frac{\norm{N_t}^2}{\norm{N_t}+\eps}
 =\ellJ\psi_\eps(\norm{N_t}).
\]
The inequalities remain valid when $N_t=0$, by the zero convention. As $g_t=N_t-r_t$,
\[
\ip{g_t}{D_t}
 =\ip{N_t}{D_t}-\ip{r_t}{D_t}
 \ge\ellJ\psi_\eps(\norm{N_t})-B\norm{r_t}.
\]
The reverse triangle inequality and the Lipschitz property yield $\psi_\eps(\norm{N_t})\ge\psi_\eps(\norm{g_t})-\norm{r_t}$, which proves \eqref{eq:inexact-align}. Applying the standard smoothness bound to $W_{t+1}=W_t-\eta_tD_t$ and then $\norm{D_t}\le B$ proves \eqref{eq:inexact-dec}. Finally, the reverse triangle inequality, $L$-smoothness, and $\norm{D_t}\le B$ imply
\[
 \bigl|\norm{g_{t+1}}-\norm{g_t}\bigr|
 \le\norm{g_{t+1}-g_t}
 \le L\eta_t\norm{D_t}\le LB\eta_t,
\]
which proves \eqref{eq:stoch-regularity-explain}.
\end{proof}

\subsubsection{Interpolation and full-sequence regularity}
\label{app:proof-lem-regularity}
\begin{lemma}[Interpolation and full-sequence regularity]\label{lem:regularity}
Let $\eps\ge0$, $h_t\ge0$, and let deterministic $\eta_t>0$ satisfy \eqref{eq:steps}. If there is a finite $C\ge0$ such that
\begin{equation}\label{eq:regass}
 \sum_{t=1}^{\infty}\eta_t\psi_\eps(h_t)<\infty,
 \quad |h_{t+1}-h_t|\le C\eta_t\quad(t\ge1),
\end{equation}
then $h_t\to0$.
\end{lemma}

\begin{proof}
Define $s_1=0$ and $s_{t+1}=s_t+\eta_t$. The first condition in \eqref{eq:steps} implies $s_t\to\infty$. On each interval $[s_t,s_{t+1}]$, define the linear interpolant
\[
 h(s)=h_t+\frac{s-s_t}{\eta_t}(h_{t+1}-h_t).
\]
The bound in \eqref{eq:regass} implies that each piece has slope at most $C$ in absolute value. Continuity across its breakpoints then implies that $h:[0,\infty)\to[0,\infty)$ is globally $C$-Lipschitz. Since $\psi_\eps$ is $1$-Lipschitz and nonnegative,
\begin{align*}
 \int_{s_t}^{s_{t+1}}\psi_\eps(h(s))\,ds
 &\le \int_{s_t}^{s_{t+1}}
   \bigl[\psi_\eps(h_t)+C(s-s_t)\bigr]\,ds\\
 &=\eta_t\psi_\eps(h_t)+\frac{C}{2}\eta_t^2.
\end{align*}
Summing and using \eqref{eq:steps} and \eqref{eq:regass} yields
\begin{equation}\label{eq:intfinite}
 \int_0^\infty \psi_\eps(h(s))\,ds<\infty.
\end{equation}
Set $f(s)=\psi_\eps(h(s))$. It is a nonnegative globally $C$-Lipschitz function. We justify explicitly the implication $f(s)\to0$: if $C=0$, then $f$ is constant and \eqref{eq:intfinite} forces that constant to be zero. If $C>0$ and there exist $d>0$ and arbitrarily large $s$ with $f(s)\ge d$, then $f(u)\ge d/2$ for all $|u-s|\le d/(2C)$. We may pick infinitely many such centers with pairwise disjoint intervals of this radius, producing an infinite integral in \eqref{eq:intfinite}, a contradiction. Hence $f(s)\to0$. In particular, $\psi_\eps(h_t)\to0$. Finally, $\psi_\eps(x)\ge\psi_\eps(d)>0$ whenever $x\ge d>0$, so $h_t\to0$.
\end{proof}

\subsubsection{Proof of Theorem~\ref{thm:criterion}}
\label{app:proof-thm-criterion}
\begin{proof}
First consider any sample path on which \eqref{eq:weighted-error-path} holds. Telescoping \eqref{eq:inexact-dec} from $1$ through $T$, and using $F(W_{T+1})\ge\Finf$, gives
\begin{align}\label{eq:weighted-psi}
 \ellJ\sum_{t=1}^T\eta_t\psi_\eps(\norm{g_t})
 &\le F(W_1)-\Finf
 +(\ellJ+B)\sum_{t=1}^T\eta_t\norm{r_t}
 +\frac{LB^2}{2}\sum_{t=1}^T\eta_t^2.
\end{align}
The right side is bounded independently of $T$, proving $\sum_t\eta_t\psi_\eps(\norm{g_t})<\infty$. Moreover the reverse triangle inequality, gradient Lipschitzness, \eqref{eq:update}, and $\norm{D_t}\le B$ give
\begin{equation*}
 \bigl|\norm{g_{t+1}}-\norm{g_t}\bigr|
 \le\norm{g_{t+1}-g_t}\le LB\eta_t.
\end{equation*}
Lemma~\ref{lem:regularity} applied to $h_t=\norm{g_t}$ proves $\norm{g_t}\to0$ along the chosen path.

For random iterates, \eqref{eq:weighted-error-mean} and Tonelli's theorem imply \eqref{eq:weighted-error-path} almost surely, so we obtain the almost-sure conclusion. To prove the mean-square conclusion, note that the same descent inequality without its negative term implies for every $t$ the pathwise estimate
\begin{equation}\label{eq:criterion-dom}
 0\le F(W_t)-\Finf\le R_A
 :=F(W_1)-\Finf
 +(\ellJ+B)\sum_{j=1}^{\infty}\eta_j\norm{r_j}
 +\frac{LB^2}{2}\sum_{j=1}^{\infty}\eta_j^2.
\end{equation}
Since $W_1$ is deterministic and \eqref{eq:weighted-error-mean} holds, $\EE R_A<\infty$. Lemma~\ref{lem:self-bound} gives $0\le\norm{g_t}^2\le2LR_A$ for all $t$. Dominated convergence now yields $\EE\norm{g_t}^2\to0$.
\end{proof}

\subsection{Polynomial-Schedule Memory and Noise Filtering}
\subsubsection{Polynomial-Schedule Convolution Estimates}
\label{app:proof-lem-convolution}
Let $1-\beta_t=\alpha_0t^{-q}$, where $0<\alpha_0\le1$ and $0<q<1$. For $1\le s\le t$, define
\begin{equation*}
  R_{s:t}=\prod_{j=s}^t\beta_j,
  \quad R_{t+1:t}=1.
\end{equation*}

\begin{lemma}[Geometric-memory estimates]\label{lem:convolution}
If $\eta_t=\eta_0t^{-p}$ with $p>0$ and $\eta_0>0$, then as $t\to\infty$,
\begin{align}
 R_{1:t}&=O\bigl(e^{-c t^{1-q}}\bigr)
      \quad\text{for some }c>0,\label{eq:conv0}\\
 \sum_{s=2}^t R_{s:t}\eta_{s-1}&=O(t^{q-p}),\label{eq:conv1}\\
 \sum_{s=1}^t (1-\beta_s)^2R_{s+1:t}^2&=O(t^{-q}).\label{eq:conv2}
\end{align}
The implied constants depend on $\alpha_0,\eta_0,p,q$ but not on $t$.
\end{lemma}

\begin{proof}
Since $1-x\le e^{-x}$ for $x\in[0,1]$,
\begin{equation}\label{eq:exp-prod}
 R_{s:t}\le\exp\left(-\sum_{j=s}^t(1-\beta_j)\right).
\end{equation}
For $s=1$, the integral lower bound $\sum_{j=1}^t j^{-q}\ge \int_1^{t+1}x^{-q}\,dx =((t+1)^{1-q}-1)/(1-q)$ proves \eqref{eq:conv0} (adjusting $c$ for small $t$).

For the two sums, write $m=\lfloor t/2\rfloor$ and first take $t\ge4$. If $s\le m$, both $R_{s:t}$ and $R_{s+1:t}$ contain the product over $j=m+1,\ldots,t$; therefore
\begin{equation*}
 R_{s:t}\le e^{-\alpha_0(t-m)t^{-q}},\quad
 R_{s+1:t}\le e^{-\alpha_0(t-m)t^{-q}}
 \le e^{-(\alpha_0/2)t^{1-q}}.
\end{equation*}
Since $\eta_{s-1}\le\eta_0$ and $(1-\beta_s)\le\alpha_0$, the contributions from $s\le m$ to \eqref{eq:conv1} and \eqref{eq:conv2} are bounded by a constant times $t e^{-c t^{1-q}}$, hence are smaller than any fixed negative power of $t$.

For $s>m$, one has $s-1\ge t/2-1\ge t/4$, and therefore
\begin{equation}\label{eq:late-params}
 \eta_{s-1}\le4^p\eta_0t^{-p},\quad
 (1-\beta_s)\le2^q\alpha_0t^{-q}.
\end{equation}
Furthermore, the sequence $(1-\beta_j)_{j\ge1}$ is nonincreasing, so $1-\beta_j\ge1-\beta_t$ for $j\le t$. Equation \eqref{eq:exp-prod} gives
\begin{equation}\label{eq:late-product}
 R_{s:t}\le e^{-(1-\beta_t)(t-s+1)},
 \quad R_{s+1:t}\le e^{-(1-\beta_t)(t-s)}.
\end{equation}
For $0<x\le1$, the inequality $1-e^{-x}\ge x/(1+x)$ implies
\begin{equation}\label{eq:geom-est}
 \sum_{k=0}^\infty e^{-xk}
 =\frac1{1-e^{-x}}\le1+\frac1x.
\end{equation}
Using \eqref{eq:late-params}--\eqref{eq:geom-est} yields
\begin{align*}
 \sum_{s=m+1}^t R_{s:t}\eta_{s-1}
 &\le4^p\eta_0t^{-p}
    \sum_{k=1}^\infty e^{-(1-\beta_t) k}
 =O(t^{-p}(1-\beta_t)^{-1})=O(t^{q-p}),\\
 \sum_{s=m+1}^t(1-\beta_s)^2R_{s+1:t}^2
 &\le 2^{2q}\alpha_0^2t^{-2q}
    \sum_{k=0}^\infty e^{-2(1-\beta_t) k}
 =O(t^{-2q}(1-\beta_t)^{-1})=O(t^{-q}).
\end{align*}
Combining the early and late estimates proves both assertions. The finitely many values $t<4$ can be absorbed into the constants.
\end{proof}

\subsubsection{Noise filtering and momentum tracking}
\label{app:proof-lem-ema-tracking}
\begin{lemma}[Noise filtering and momentum tracking]\label{lem:ema-tracking}
Under Assumption~\ref{ass:smooth}, \eqref{eq:MDS}, and the polynomial schedules~\eqref{eq:schedule} with $0<q<1$ and $p>q$, let $e_t=M_t-g_t$. For $\theta_t=1$ or $\theta_t=\beta_t$,
\begin{align}
 \EE\norm{e_t}&=O\bigl(t^{-q/2}+t^{q-p}\bigr),\label{eq:track-rate}\\
 \EE\norm{r_t}&=O\bigl(t^{-q/2}+t^{q-p}\bigr).\label{eq:r-rate}
\end{align}
In particular, the exponent conditions~\eqref{eq:exponents} imply
\begin{equation}\label{eq:weighted-tracking-lemma}
 \sum_{t=1}^{\infty}\eta_t\EE\norm{r_t}<\infty.
\end{equation}
\end{lemma}

\begin{proof}
Set $e_t=M_t-g_t$. The time-varying counterpart of \eqref{eq:erecur} is
\[
 e_t=\beta_t(e_{t-1}+g_{t-1}-g_t)+(1-\beta_t)\zeta_t,
 \quad t\ge2,
\]
while at $t=1$, $e_1=\beta_1(M_0-g_1)+(1-\beta_1)\zeta_1$. Iteration gives the \emph{exact}, consistently indexed identity
\begin{equation}\label{eq:ema-expansion}
 e_t=R_{1:t}(M_0-g_1)
 +\sum_{s=2}^t R_{s:t}(g_{s-1}-g_s)
 +\sum_{s=1}^t(1-\beta_s)R_{s+1:t}\zeta_s.
\end{equation}
The initial term is $O(e^{-c t^{1-q}})$ by Lemma~\ref{lem:convolution}. Since $\norm{D_{s-1}}\le B$,
\begin{equation}\label{eq:drift}
 \norm{g_{s-1}-g_s}
 \le L\norm{W_s-W_{s-1}}\le LB\eta_{s-1}.
\end{equation}
Consequently the norm of the middle sum in \eqref{eq:ema-expansion} is bounded \emph{pathwise} by $LB\sum_{s=2}^t R_{s:t}\eta_{s-1}=O(t^{q-p})$.

For the final sum in \eqref{eq:ema-expansion}, define the \emph{deterministic} scalars $w_{s,t}=(1-\beta_s)R_{s+1:t}$. Because $\zeta_s$ is $\calF_s$-measurable and $\EE[\zeta_u\mid\calF_{u-1}]=0$, for $s<u$,
\[
 \EE\ip{\zeta_s}{\zeta_u}
 =\EE\Bigl[\ip{\zeta_s}{
     \EE[\zeta_u\mid\calF_{u-1}]}\Bigr]=0.
\]
The interchange is legitimate by Cauchy--Schwarz and the bounded second moments. Expanding the square and using \eqref{eq:MDS} therefore gives
\begin{align*}
 \EE\norm{\sum_{s=1}^t w_{s,t}\zeta_s}^2
 &=\sum_{s=1}^t w_{s,t}^2\EE\norm{\zeta_s}^2\\
 &\le\sigma^2\sum_{s=1}^t(1-\beta_s)^2R_{s+1:t}^2
 =O(t^{-q}).
\end{align*}
Jensen's inequality yields an expected norm of order $t^{-q/2}$. Applying the triangle inequality to the three terms of \eqref{eq:ema-expansion} proves \eqref{eq:track-rate}.

If $\theta_t=1$, then $r_t=e_t$. If $\theta_t=\beta_t$, the exact identity \eqref{eq:mix} gives
\[
 r_t=\beta_te_t+(1-\beta_t)\zeta_t.
\]
Thus $\EE\norm{r_t}\le\EE\norm{e_t}+(1-\beta_t)\sigma$; $1-\beta_t=O(t^{-q})$ is absorbed in $O(t^{-q/2})$. This proves \eqref{eq:r-rate}. Under \eqref{eq:exponents}, multiplying the estimate by $\eta_t=\eta_0t^{-p}$ gives
\[
 \eta_t\EE\norm{r_t}
 =O(t^{-p-q/2}+t^{q-2p}).
\]
Both exponents are less than $-1$ because $p+q/2>1$ and $2p-q>1$. Summation proves \eqref{eq:weighted-tracking-lemma}.
\end{proof}

\subsection{Proof of Theorem~\ref{thm:stoch}}
\label{app:proof-thm-stoch}
\begin{proof}
The interval for $q$ in \eqref{eq:exponents} is nonempty exactly when $p>3/4$. Because $p\le1$, it also implies $0<q<1$; because $q<2p-1\le p$ for $p\le1$, we have $q<p$. Lemma~\ref{lem:ema-tracking} therefore applies and directly gives \eqref{eq:weighted-tracking-lemma}. Also $p>3/4>1/2$ and $p\le1$ imply $\sum_t\eta_t^2<\infty$ and $\sum_t\eta_t=\infty$. All hypotheses of Theorem~\ref{thm:criterion} are satisfied; applying it proves both statements in \eqref{eq:stoch-conclusion}.

\emph{Almost-sure weighted square-summability.} The argument leading to \eqref{eq:criterion-dom} yields an almost-surely finite random variable $R_A\ge0$ such that
\[
 \sup_{t\ge1}\bigl(F(W_t)-\Finf\bigr)\le R_A,
 \quad
 \ellJ\sum_{t=1}^{\infty}\eta_t
      \psi_\eps(\norm{g_t})\le R_A
 \quad\as.
\]
(The latter inequality follows directly from \eqref{eq:weighted-psi}, and $R_A<\infty$ follows from $\sum_t\eta_t\norm{r_t}<\infty$ almost surely.) The gradient self-bound \eqref{eq:selfbound} implies $\norm{g_t}\le\sqrt{2LR_A}$ simultaneously for all $t$, almost surely. By the exact identity $x^2=(x+\eps)\psi_\eps(x)$ for every $x\ge0$, including $\eps=x=0$ under our convention, we obtain
\begin{align*}
 \sum_{t=1}^{\infty}\eta_t\norm{g_t}^2
 &=\sum_{t=1}^{\infty}\eta_t
       (\norm{g_t}+\eps)\psi_\eps(\norm{g_t})\\
 &\le\frac{(\sqrt{2LR_A}+\eps)R_A}{\ellJ}
 <\infty\quad\as.
\end{align*}
This proves \eqref{eq:stoch-weighted-main} and completes Theorem~\ref{thm:stoch}.
\end{proof}

\subsection{Fixed Momentum with Summably Weighted Noise}
\begin{corollary}[Fixed momentum with summably weighted noise]
\label{cor:fixed-diminishing}
Under Assumption~\ref{ass:smooth}, fix \(\eps\ge0\), \(J\ge1\), \(\beta\in[0,1)\), and \(\theta\in[0,1]\). Let \(\eta_t\) be positive, nonincreasing, and satisfy \eqref{eq:steps}. If deterministic \(\sigma_t\ge0\) satisfy
\begin{equation}\label{eq:sigma}
 \EE\norm{\zeta_t}\le\sigma_t,\quad
 \sum_{t\ge1}\eta_t\sigma_t<\infty,
\end{equation}
then \(\norm{g_t}\to0\) almost surely and \(\EE\norm{g_t}^2\to0\), without unbiasedness.
\end{corollary}
For example, \(\eta_t\asymp t^{-4/5}\) and \(\EE\norm{\zeta_t}^2\lesssim t^{-1/2}\) satisfy this condition.

\subsubsection{Proof of Corollary~\ref{cor:fixed-diminishing}}
\label{app:proof-cor-fixed-diminishing}
\begin{proof}
With $e_t=M_t-g_t$, expansion of the constant-$\beta$ recursion shows that for all $t\ge1$,
\begin{equation*}
 e_t=\beta^t(M_0-g_1)
 +\sum_{s=2}^t\beta^{t-s+1}(g_{s-1}-g_s)
 +\alpha\sum_{s=1}^t\beta^{t-s}\zeta_s,
 \quad\alpha=1-\beta.
\end{equation*}
Take expected norms, multiply by $\eta_t$, and sum. For the initial condition, the monotonicity of $\eta_t$ yields
\begin{equation}\label{eq:fixconv-init}
 \sum_{t=1}^{\infty}\eta_t\beta^t\norm{M_0-g_1}
 \le\frac{\beta\eta_1}{1-\beta}\norm{M_0-g_1}.
\end{equation}
For the drift part, use \eqref{eq:drift}, exchange the order of nonnegative sums, and note that if $t\ge s$ then $\eta_t\le\eta_{s-1}$. Therefore
\begin{align}\label{eq:fixconv-drift}
 &\sum_{t\ge1}\eta_t\sum_{s=2}^t
     \beta^{t-s+1}\norm{g_{s-1}-g_s}\\
 &\quad\le LB\sum_{s=2}^{\infty}\eta_{s-1}
      \sum_{t=s}^{\infty}\eta_t\beta^{t-s+1}\notag\\
 &\quad\le\frac{LB\beta}{1-\beta}
      \sum_{s=2}^{\infty}\eta_{s-1}^2
  =\frac{LB\beta}{1-\beta}\sum_{t\ge1}\eta_t^2<\infty.
 \notag
\end{align}
Finally, by \eqref{eq:sigma}, nonincreasing steps, and $\alpha\sum_{k\ge0}\beta^k=1$,
\begin{align}\label{eq:fixconv-noise}
 &\alpha\sum_{t\ge1}\eta_t
    \sum_{s=1}^t\beta^{t-s}\EE\norm{\zeta_s}\\
 &\quad\le\alpha\sum_{s\ge1}\sigma_s
    \sum_{t=s}^{\infty}\eta_t\beta^{t-s}
 \le\sum_{s\ge1}\eta_s\sigma_s<\infty.
 \notag
\end{align}
Equations \eqref{eq:fixconv-init}--\eqref{eq:fixconv-noise} show $\sum_t\eta_t\EE\norm{e_t}<\infty$. Since $r_t=\theta e_t+(1-\theta)\zeta_t$, \eqref{eq:sigma} implies $\sum_t\eta_t\EE\norm{r_t}<\infty$. Theorem~\ref{thm:criterion} proves the claim. For $\beta=0$, the terms in \eqref{eq:fixconv-init}--\eqref{eq:fixconv-drift} vanish and the same proof is valid with $r_t=\zeta_t$.
\end{proof}

\subsubsection{Fixed-Memory Stochastic Bias and Growing-Memory Filtering}
Do, Dereich, and Jentzen~\cite[Theorem~1.3]{do2026} identify a complementary stochastic regime for generalized finite-NS Muon: asymmetric data can induce a persistent bias with fixed memory. The following specialized statement records the update and noise conditions relevant to our comparison.

\begin{lemma}[Fixed momentum with decreasing steps: a nonconvergence example
\cite{do2026}]\label{lem:fixed-momentum-obstruction}
There are bounded i.i.d. data $X$ with nonzero third centered moment, a one-dimensional strongly convex quadratic objective $F(w)=\mathbb E[(w-X)^2]$, and fixed-batch instances of finite-NS Muon with the practical polynomial, positive input stabilization, EMA input ($\theta_t=1$), fixed $\beta\in(0,1)$, and decreasing steps $\eta_t=\eta_0t^{-p}$ with $0<p<1$, for which
\[
 \limsup_{t\to\infty}
 \mathbb E\!\left[\min\{1,|w_t-\mathbb E X|\}\right]>0.
\]
Consequently, the prescribed decreasing learning rates and fixed momentum need not yield last-iterate stationarity.
\end{lemma}

Our stochastic theorem uses $1-\beta_t\downarrow0$ to increase the noise-averaging window. The example therefore separates the fixed-memory and growing-memory noise-filtering regimes; variance reduction provides another approach to noise control~\cite{qian2025}.

%% file: appendices/d-density-one-proofs.tex
\section{Proofs for Section~\ref{sec:rates}: Density-One Rates}\label{app:rate-proofs}
The proofs are organized in the order of Section~\ref{sec:rates}: the sparse-peak counterexample, the general density-one counting argument, and the rate theorem. The perturbed fixed-step variant follows.

\subsection{Proof of the Common Sparse-Peak Construction}
\label{app:pointwise-proof}
\begin{proof}[Proof of Proposition~\ref{prop:stoch-no-rate}]
The bumps in \eqref{eq:sparse-stoch-grad} have disjoint supports, height $1/k$, and width $O(1/k)$. Because $|x_{t_k}|\asymp e^{k^2/5}$, their supports are disjoint for sufficiently large $k$. Thus $H>0$ is smooth with $\sup|H'|<\infty$, and $\int_{\RR}H=\pi+(\int\rho)\sum_{k\ge k_0}k^{-2}<\infty$. Consequently $F$ is globally smooth and lower bounded, with $H(x)\to0$ as $x\to-\infty$.

In either regime, positive gradients give $M_t,N_t,D_t>0$, so $w_t$ decreases. If $w_t\to w_*>-\infty$, then $H(w_t)\to h_*=H(w_*)>0$ and $M_t,N_t\to h_*$: for fixed momentum this follows from $\beta<1$, and for increasing momentum from $\sum_t(1-\beta_t)=\infty$. Hence $D_t\to\Phi_\eps(h_*)>0$, contradicting $\sum_t\eta_t=\infty$. Therefore $w_t\downarrow-\infty$.

The first-crossing time $s_k$ defined in Section~\ref{subsec:rate-separation} is finite for all large $k$. Since $0<D_t\le\uJ$ and the learning rates are bounded, reaching $x_{t_k}$ requires $s_k\ge c e^{k^2/5}$ for some $c>0$. If the crossing step is shorter than $1/k$, then
\begin{equation}\label{eq:fixed-peak-hit}
 |w_{s_k}-x_{t_k}|\le\frac1{2k},\quad
 H(w_{s_k})\ge\frac1k.
\end{equation}
We verify this step bound in both regimes.

\emph{Fixed steps.} Define the first exit time from the $(k-1)$st bump by $\tau_k:=\min\{t\ge2:w_t\le x_{t_{k-1}}-1/(k-1)\}$. The uniform step bound and $|x_{t_k}-x_{t_{k-1}}|\asymp e^{k^2/5}$ imply $s_k-\tau_k\ge c_1e^{k^2/5}$ for large $k$. Between $\tau_k$ and $s_k-1$, the earlier bumps have been passed and $H(w_t)\le2/k$. Hence, with $H_{\max}=\sup H$,
\[
 M_{s_k-1}\le H_{\max}\beta^{s_k-\tau_k}+\frac2k
 \le\frac3k,
 \quad N_{s_k-1}\le\frac3k.
\]
The finite-NS bound and the admissible learning rate give
\[
 w_{s_k-1}-w_{s_k}\le
 \frac{\eta\uJ}{\eps}N_{s_k-1}
 \le\frac{3(1-\beta)}{4Lk}<\frac1k.
\]
Thus \eqref{eq:fixed-peak-hit} holds.

\emph{Diminishing steps.} The bound on $s_k$ and $\eta_t=\eta_0t^{-4/5}$ yield
\[
 w_{s_k-1}-w_{s_k}\le\eta_0\uJ(s_k-1)^{-4/5}
 =o(k^{-1}),
\]
so \eqref{eq:fixed-peak-hit} again holds.

In either case, for every $\nu>0$,
\[
 s_k^\nu|F'(w_{s_k})|\ge\frac{c^\nu e^{\nu k^2/5}}k
 \longrightarrow\infty.
\]
The fixed-step parameters satisfy Theorem~\ref{thm:fixed}; the second regime satisfies Theorem~\ref{thm:stoch} since $p=4/5$, $q=1/2$, and $\zeta_t=0$. Both theorems give $|F'(w_t)|\to0$, completing the proof.
\end{proof}

\subsection{A General Summability-to-Density Principle}
\begin{lemma}[Summability implies a density-one rate]\label{lem:density-general}
Let \(a_t\ge0\), \(0\le p<1\), and \(s>0\). If
\begin{equation*}
 \sum_{t=1}^\infty t^{-p}a_t^s<\infty,
\end{equation*}
then
\begin{equation}\label{eq:generic-dens-concl}
 a_t=o_{\mathrm{dens}}\bigl(t^{-(1-p)/s}\bigr).
\end{equation}
\end{lemma}

\subsubsection{Proof of Lemma~\ref{lem:density-general}}
\label{app:proof-density-general}
\begin{proof}
Set $b_t=t^{-p}a_t^s\ge0$ and $B_*:=\sum_{t=1}^\infty b_t<\infty$. For any $c>0$, define the bad set $\mathcal B_c:=\{t:a_t>c\,t^{-(1-p)/s}\}$. For $t\in\mathcal B_c$, $b_t>c^s/t$, so
\begin{equation*}
 \frac{\#(\mathcal B_c\cap[1,T])}{T}
 \le\frac1{c^s T}\sum_{t=1}^T t b_t.
\end{equation*}
The right side tends to zero: given any integer $N<T$,
\[
 \frac1T\sum_{t=1}^T t b_t
 \le\frac1T\sum_{t=1}^N t b_t
      +\sum_{t>N} b_t.
\]
First let $T\to\infty$, then $N\to\infty$. This proves \eqref{eq:generic-dens-concl}.

\end{proof}

\subsection{Proof of Theorem~\ref{thm:density-rates}}
\label{app:proof-density-theorem}
\begin{proof}
For the fixed-step exact-gradient regime, Theorem~\ref{thm:fixed} gives $\sum_{t=1}^{\infty}\norm{g_t}^2<\infty$. Lemma~\ref{lem:density-general}, applied with $p=0$, $s=2$, and $a_t=\norm{g_t}$, yields \eqref{eq:dens-fixed}. The square-summable-error extension is stated separately in Corollary~\ref{cor:perturbed-density}.

For the stochastic regime, Theorem~\ref{thm:stoch} gives $\sum_t\eta_t\norm{g_t}^2<\infty$ almost surely. If $p<1$ and $\eps>0$, substitute $\eta_t=\eta_0t^{-p}$ and apply Lemma~\ref{lem:density-general} with $s=2$. If $\eps=0$, then $\psi_0(x)=x$ and \eqref{eq:stoch-weighted-main} yield $\sum_t t^{-p}\norm{g_t}<\infty$ almost surely. The same lemma with $s=1$ gives the second rate in \eqref{eq:stoch-density-rates}.
\end{proof}

\subsection{Density-One Rates under Square-Summable Perturbations}
\begin{corollary}[Perturbed fixed-step density-one rate]
\label{cor:perturbed-density}
Under the pathwise conditions of Theorem~\ref{thm:fixed-perturbed}, \(\norm{g_t}=o_{\mathrm{dens}}(t^{-1/2})\). Under \eqref{eq:integrated-noise}, the rate holds almost surely.
\end{corollary}
\begin{proof}
Apply Lemma~\ref{lem:density-general} with \(p=0\), \(s=2\), and \(a_t=\norm{g_t}\), using \eqref{eq:perturbed-result}.
\end{proof}

%% file: references.bib
@inproceedings{bernstein2024,
  author = {Jeremy Bernstein and Laker Newhouse},
  title = {{Old Optimizer, New Norm: An Anthology}},
  booktitle = {NeurIPS 2024 Workshops: OPT},
  year = {2024},
  url = {https://openreview.net/forum?id=ux18f5nOpD},
}

@inproceedings{carlson2015,
  author = {David Carlson and Volkan Cevher and Lawrence Carin},
  title = {{Stochastic Spectral Descent for Restricted Boltzmann Machines}},
  booktitle = {Proceedings of the Eighteenth International Conference on Artificial Intelligence and Statistics},
  series = {Proceedings of Machine Learning Research},
  volume = {38},
  pages = {111--119},
  year = {2015},
  url = {https://proceedings.mlr.press/v38/carlson15.html},
}

@misc{chang2026,
  author = {Da Chang and Yongxiang Liu and Huihui Zhou and Ganzhao Yuan},
  title = {{On the Convergence of Muon and Beyond}},
  year = {2026},
  howpublished = {arXiv preprint arXiv:2509.15816v6},
  url = {https://arxiv.org/abs/2509.15816v6},
  note = {First posted in 2025},
}

@misc{choudhury2026,
  author = {Sayantan Choudhury and Xiaoran Cheng and Martin Tak\'a\v{c} and Sen Na and Mladen Kolar},
  title = {{Muon with Nesterov Momentum: Heavy-Tailed Noise and (Randomized) Inexact Polar Decomposition}},
  year = {2026},
  howpublished = {arXiv preprint arXiv:2605.06884},
  url = {https://arxiv.org/abs/2605.06884},
}

@inproceedings{cutkosky2020,
  author = {Ashok Cutkosky and Harsh Mehta},
  title = {Momentum Improves Normalized {SGD}},
  booktitle = {Proceedings of the 37th International Conference on Machine Learning},
  series = {Proceedings of Machine Learning Research},
  volume = {119},
  pages = {2260--2268},
  year = {2020},
  url = {https://proceedings.mlr.press/v119/cutkosky20b.html},
}

@misc{demidovich2026,
  author = {Yury Demidovich and Abhishek Chakraborty and Grigory Malinovsky and Angelia Nedi\'c and Peter Richt\'arik},
  title = {{Distance-Aware Muon: Adaptive Step Scaling for Normalized Optimization}},
  year = {2026},
  howpublished = {arXiv preprint arXiv:2605.18999v1},
  url = {https://arxiv.org/abs/2605.18999v1},
}

@misc{do2026,
  author = {Thang Do and Steffen Dereich and Arnulf Jentzen},
  title = {{On MUON optimization: From non-convergence to an error analysis with Polar Express and the Newton--Schulz polynomial from implementations}},
  year = {2026},
  howpublished = {arXiv preprint arXiv:2608.04607v1},
  url = {https://arxiv.org/abs/2608.04607v1},
}

@article{ghadimi2013,
  author = {Saeed Ghadimi and Guanghui Lan},
  title = {{Stochastic First- and Zeroth-Order Methods for Nonconvex
Stochastic Programming}},
  journal = {SIAM Journal on Optimization},
  volume = {23},
  number = {4},
  pages = {2341--2368},
  year = {2013},
  doi = {10.1137/120880811},
  url = {https://doi.org/10.1137/120880811},
}

@misc{gonon2026,
  author = {Antoine Gonon and Andreea-Alexandra Mu\c{s}at and Nicolas Boumal},
  title = {{Insights on Muon from Simple Quadratics}},
  year = {2026},
  howpublished = {arXiv preprint arXiv:2602.11948v1},
  url = {https://arxiv.org/abs/2602.11948v1},
}

@inproceedings{gruntkowska2025,
  author = {Kaja Gruntkowska and Alexander Gaponov and Zhirayr Tovmasyan and Peter Richt\'arik},
  title = {{Error Feedback for Muon and Friends}},
  booktitle = {International Conference on Learning Representations (ICLR)},
  year = {2026},
  url = {https://proceedings.iclr.cc/paper_files/paper/2026/hash/983876577ec81db17ecfae1521df9208-Abstract-Conference.html},
}

@inproceedings{gupta2018,
  author = {Vineet Gupta and Tomer Koren and Yoram Singer},
  title = {{Shampoo: Preconditioned Stochastic Tensor Optimization}},
  booktitle = {Proceedings of the 35th International Conference on Machine Learning},
  series = {Proceedings of Machine Learning Research},
  volume = {80},
  pages = {1842--1850},
  year = {2018},
  url = {https://proceedings.mlr.press/v80/gupta18a.html},
}

@inproceedings{huangLiMuon2026,
  author = {Feihu Huang and Yuning Luo and Songcan Chen},
  title = {{LiMuon}: Light and Fast Muon Optimizer for Large Models},
  booktitle = {Proceedings of the 43rd International Conference on Machine Learning},
  series = {Proceedings of Machine Learning Research},
  volume = {306},
  pages = {47263--47305},
  year = {2026},
  url = {https://proceedings.mlr.press/v306/huang26aj.html},
}

@misc{jordanBlog,
  author = {Keller Jordan and Yuchen Jin and Vlado Boza and Jiacheng You and Franz Cesista and Laker Newhouse and Jeremy Bernstein},
  title = {{Muon: An optimizer for hidden layers in neural networks}},
  year = {2024},
  howpublished = {Technical design note (originally published December 8, 2024)},
  url = {https://kellerjordan.github.io/posts/muon/},
}

@misc{jordanCode,
  author = {Jordan, Keller and others},
  title = {{Muon} reference implementation, \texttt{muon.py}},
  year = {2026},
  howpublished = {Source code repository on GitHub},
  url = {https://github.com/KellerJordan/Muon/blob/master/muon.py},
  note = {Accessed October 17, 2026; mutable master branch},
}

@misc{kimiK22025,
  author = {{Kimi Team}},
  title = {{Kimi K2: Open Agentic Intelligence}},
  year = {2025},
  howpublished = {arXiv preprint arXiv:2507.20534},
  url = {https://arxiv.org/abs/2507.20534},
}

@inproceedings{kim2026,
  author = {Gyu Yeol Kim and {Min-hwan} Oh},
  title = {{Convergence of Muon with Newton--Schulz}},
  booktitle = {International Conference on Learning Representations (ICLR)},
  year = {2026},
  url = {https://proceedings.iclr.cc/paper_files/paper/2026/hash/4c0129f1f7a036af5ed2409f1f9f8131-Abstract-Conference.html},
}

@article{liangXu2026,
  author = {Yuqing Liang and Dongpo Xu},
  title = {{A Unified Theoretical Framework for the Last-Iterate Convergence of Stochastic Adaptive Optimization}},
  journal = {Artificial Intelligence},
  volume = {353},
  pages = {104504},
  year = {2026},
  doi = {10.1016/j.artint.2026.104504},
}

@misc{liHong2025,
  author = {Jiaxiang Li and Mingyi Hong},
  title = {{A Note on the Convergence of Muon}},
  year = {2025},
  howpublished = {arXiv preprint arXiv:2502.02900},
  url = {https://arxiv.org/abs/2502.02900},
}

@misc{liTsuchiya2026,
  author = {Mingyi Li and Taira Tsuchiya},
  title = {{Muon with Finite Newton--Schulz: The Smoothing Benefit in Nonsmooth Nonconvex Optimization}},
  year = {2026},
  howpublished = {arXiv preprint arXiv:2608.26288},
  url = {https://arxiv.org/abs/2608.26288},
}

@article{liu2024,
  author = {Jun Liu and Ye Yuan},
  title = {{Almost Sure Convergence Rates Analysis and Saddle Avoidance of Stochastic Gradient Methods}},
  journal = {Journal of Machine Learning Research},
  volume = {25},
  number = {271},
  pages = {1--40},
  year = {2024},
  url = {https://jmlr.org/papers/v25/23-1436.html},
}

@misc{liuScalable2025,
  author = {Jingyuan Liu and others},
  title = {{Muon is Scalable for LLM Training}},
  year = {2025},
  howpublished = {arXiv preprint arXiv:2502.16982},
  url = {https://arxiv.org/abs/2502.16982},
}

@misc{nagashimaIiduka2026,
  author = {Shuntaro Nagashima and Hideaki Iiduka},
  title = {{Improved Convergence Rates of Muon Optimizer for Nonconvex Optimization}},
  year = {2026},
  howpublished = {arXiv preprint arXiv:2601.19400v3},
  url = {https://arxiv.org/abs/2601.19400},
}

@inproceedings{oik2025,
  author = {Konstantinos Oikonomidis and Jan Quan and Panagiotis Patrinos},
  title = {{Nonlinearly Preconditioned Gradient Methods: Momentum and Stochastic Analysis}},
  booktitle = {Advances in Neural Information Processing Systems},
  volume = {38},
  pages = {43584--43615},
  year = {2025},
  doi = {10.52202/085713-1300},
  url = {https://proceedings.nips.cc/paper_files/paper/2025/hash/3768bb822e5c024b408dd930c4aafe67-Abstract-Conference.html},
}

@inproceedings{oik2026,
  author = {Oikonomidis, Konstantinos and Quan, Jan and Antonakopoulos, Kimon and Silveti-Falls, Antonio and Cevher, Volkan and Patrinos, Panagiotis},
  title = {Constrained Stochastic Spectral Preconditioning Converges for Nonconvex Objectives},
  booktitle = {Advances in Neural Information Processing Systems},
  year = {2026},
  eprint = {2605.11850},
  archivePrefix = {arXiv},
  url = {https://arxiv.org/abs/2605.11850},
  note = {Accepted to NeurIPS 2026},
}

@misc{oliveira2026,
  author = {Arthur C. B. de Oliveira and Dhruv D. Jatkar and Guilherme S. Vicinansa and Eduardo D. Sontag},
  title = {{Convergence guarantees for Muon: New parameter regimes and generalizations}},
  year = {2026},
  howpublished = {arXiv preprint arXiv:2609.30546v1},
  url = {https://arxiv.org/abs/2609.30546v1},
}

@misc{parshakova2026,
  author = {Tetiana Parshakova and Ahmed Khaled and Michael Crawshaw and Guillaume Garrigos and Robert M. Gower},
  title = {{Muon Does Not Converge on Convex Lipschitz Functions}},
  year = {2026},
  howpublished = {arXiv preprint arXiv:2605.08980},
  url = {https://arxiv.org/abs/2605.08980},
}

@misc{peng2026,
  author = {Hanyang Peng and Hui Wang and Yue Yu},
  title = {{Convergence of Practical Muon with Finite Newton--Schulz Iterations and Nesterov Momentum}},
  year = {2026},
  howpublished = {arXiv preprint arXiv:2609.39595v1},
  url = {https://arxiv.org/abs/2609.39595v1},
}

@inproceedings{pethickLMO2025,
  author = {Thomas Pethick and Wanyun Xie and Kimon Antonakopoulos and Zhenyu Zhu and Antonio Silveti-Falls and Volkan Cevher},
  title = {Training Deep Learning Models with Norm-Constrained {LMOs}},
  booktitle = {Proceedings of the 42nd International Conference on Machine Learning},
  series = {Proceedings of Machine Learning Research},
  volume = {267},
  pages = {49069--49104},
  year = {2025},
  url = {https://proceedings.mlr.press/v267/pethick25a.html},
}

@misc{pytorchMuon,
  author = {{PyTorch contributors}},
  title = {{Muon optimizer documentation and reference implementation}},
  year = {2026},
  howpublished = {PyTorch online documentation},
  url = {https://docs.pytorch.org/docs/main/generated/torch.optim.Muon.html},
  note = {Accessed October 18, 2026; mutable main documentation},
}

@misc{qian2025,
  author = {Xun Qian and Hussein Rammal and Dmitry Kovalev and Peter Richt\'arik},
  title = {{Muon is Provably Faster with Momentum Variance Reduction}},
  year = {2025},
  howpublished = {arXiv preprint arXiv:2512.16598},
  url = {https://arxiv.org/abs/2512.16598},
}

@inproceedings{riabinin2026,
  author = {Artem Riabinin and Egor Shulgin and Kaja Gruntkowska and Peter Richt\'arik},
  title = {From Muon to Gluon: Bridging Theory and Practice of {LMO}-based Optimizers for {LLMs}},
  booktitle = {Proceedings of the 43rd International Conference on Machine Learning},
  series = {Proceedings of Machine Learning Research},
  volume = {306},
  pages = {105039--105082},
  year = {2026},
  url = {https://proceedings.mlr.press/v306/riabinin26a.html},
}

@misc{royChakraborty2026,
  author = {Amartya Roy and Souvik Chakraborty},
  title = {{Training-Loss Guarantees for Muon with Finite-Step Newton--Schulz Orthogonalization}},
  year = {2026},
  howpublished = {arXiv preprint arXiv:2610.03306},
  url = {https://arxiv.org/abs/2610.03306},
}

@misc{sato2025,
  author = {Naoki Sato and Hiroki Naganuma and Hideaki Iiduka},
  title = {{Convergence Bound and Critical Batch Size of Muon Optimizer}},
  year = {2025},
  howpublished = {arXiv preprint arXiv:2507.01598},
  url = {https://arxiv.org/abs/2507.01598},
  note = {Revised in 2026},
}

@article{shen2025,
  author = {Wei Shen and Ruichuan Huang and Minhui Huang and Cong Shen and Jiawei Zhang},
  title = {{On the Convergence Analysis of Muon}},
  journal = {Transactions on Machine Learning Research},
  year = {2026},
  url = {https://openreview.net/forum?id=4nH4CulGaP},
}

@inproceedings{shulgin2026,
  author = {Egor Shulgin and Mohamed Awad and Peter Richt\'arik and Eduard Gorbunov},
  title = {General Analysis of {LMO}-based Optimizers: Beyond Bounded Variance},
  booktitle = {Proceedings of the 43rd International Conference on Machine Learning},
  series = {Proceedings of Machine Learning Research},
  volume = {306},
  pages = {112893--112935},
  year = {2026},
  url = {https://proceedings.mlr.press/v306/shulgin26a.html},
}

@inproceedings{shulginInexact2026,
  author = {Egor Shulgin and Sultan AlRashed and Peter Richt\'{a}rik and Francesco Orabona},
  title = {{Beyond the Ideal: Analyzing the Inexact Muon Update}},
  booktitle = {Proceedings of the 29th International Conference on Artificial Intelligence and Statistics},
  series = {Proceedings of Machine Learning Research},
  volume = {300},
  pages = {5167--5175},
  year = {2026},
  url = {https://proceedings.mlr.press/v300/shulgin26a.html},
}

@misc{vernimmen2026,
  author = {Pierre Vernimmen and Fran\c{c}ois Glineur},
  title = {{Can Acceleration in Gradient-Norm Minimization Be Anytime?
Sharp Last-Iterate Limits in Smooth Convex Optimization}},
  year = {2026},
  howpublished = {arXiv preprint arXiv:2609.05710v1},
  url = {https://arxiv.org/abs/2609.05710},
}

@misc{wang2026,
  author = {Haonan Wang and Yu Wu and Minghui Liwang and Xinlei Yi and Yiguang Hong},
  title = {{Convergence of Practical Muon}},
  year = {2026},
  howpublished = {arXiv preprint arXiv:2609.33152v1},
  url = {https://arxiv.org/abs/2609.33152v1},
}
